\documentclass[a4paper,oneside]{article}

\usepackage{tikz}
\usetikzlibrary{arrows,patterns,plotmarks,shapes,snakes,er,3d,automata,backgrounds,topaths,trees,petri,mindmap,decorations.pathreplacing}
\tikzstyle{vertex}=[circle,fill=black!25,minimum size=20pt,inner sep=0pt]
\tikzstyle{selected vertex} = [vertex, fill=red!24]
\tikzstyle{edge} = [draw,thick,-]
\tikzstyle{weight} = [font=\small]
\tikzstyle{selected edge} = [draw,line width=5pt,-,red!50]
\tikzstyle{ignored edge} = [draw,line width=5pt,-,black!20]

\usepackage{algorithm,algorithmicx}
\usepackage{algpseudocode}
\usepackage{xspace}
\usepackage{amsmath,amsfonts}
\usepackage{url}
\usepackage{graphicx}    
\usepackage{amsthm}
\usepackage{geometry}
\newcommand{\PLATEAU}{\textsc{Plateau}\xspace}
\newcommand{\GAP}{\textsc{Gap}\xspace}
\newcommand{\LINEAR}{\textsc{Linear\xspace}}
\newcommand{\UNITATION}{\textsc{Unitation\xspace}}
\newcommand{\ONEMAX}{\textsc{Onemax}\xspace}

\newcommand{\LEADINGONES}{\textsc{LeadingOnes}\xspace}

\newcommand{\NEEDLE}{\textsc{Needle}\xspace}

\newcommand{\xmin}{\ensuremath{x_{\min}}}
\newcommand{\xmax}{\ensuremath{x_{\max}}}
\newcommand{\wmin}{\ensuremath{w_{\min}}}
\newcommand{\wmax}{\ensuremath{w_{\max}}}

\newcommand{\expect}[1]{\mathbb{E}\left[#1\right]}

\newcommand{\Prob}[1]{\Pr\left(#1\right)}

\newtheorem{theorem}             {Theorem}
\newtheorem{lemma}      [theorem]{Lemma}

\newtheorem{definition} [theorem]{Definition}

\definecolor{darkgreen}{rgb}{0.11, 0.44, 0.11}
\DeclareMathOperator{\bin}{Bin}

\begin{document}

\title{Theoretical Analysis of Stochastic Search Algorithms}


\author{Per Kristian Lehre\\
\small School of Computer Science,\\
\small University of Birmingham,\\
\small  Birmingham, UK
\and Pietro S. Oliveto\\
\small Department of Computer Science,\\
\small University of Sheffield,\\
\small  Sheffield, UK
}

\maketitle

\abstract{Theoretical analyses of stochastic search algorithms, albeit few, have always existed since these algorithms became popular. 
Starting in the nineties a systematic approach to analyse the performance of stochastic search heuristics has been put in place.
This quickly increasing basis of results allows, nowadays, the analysis of sophisticated algorithms such as population-based evolutionary algorithms, 
ant colony optimisation and artificial immune systems. 
Results are available concerning problems from various domains including classical combinatorial and continuous optimisation, single and multi-objective
optimisation, and noisy and dynamic optimisation.
This chapter introduces the mathematical techniques that are most commonly used in the runtime analysis of stochastic search heuristics.
Careful attention is given to the very popular artificial fitness levels and drift analyses techniques for which several variants are presented.
To aid the reader's comprehension of the presented mathematical methods, these are applied to the analysis of simple evolutionary algorithms for artificial example functions.
The chapter is concluded by providing references to more complex applications and further extensions of the techniques for the obtainment of advanced results.
}

%

\section{Introduction}

Stochastic search algorithms, also called randomised search
heuristics, are general purpose optimisation algorithms that are often
used when it is not possible to design a specific algorithm for the
problem at hand.  Common reasons are the lack of available resources
(e.g., enough money and/or time) or because of an insufficient
knowledge of the complex optimisation problem which has not been
studied extensively before. Other times, the only way of acquiring
knowledge about the problem is by evaluating the quality of candidate
solutions.

Well-known stochastic search algorithms are random local search and simulated annealing. Other
more complicated approaches are inspired by processes observed in
nature. Popular examples are evolutionary algorithms (EAs)
 inspired by the concept of natural evolution, ant
colony optimisation (ACO) inspired by ant
foraging behaviour and artificial immune systems (AIS) inspired by the
immune system of vertebrates.

The main advantage of stochastic search heuristics is that, being
general purpose algorithms, they can be applied to a wide range of
applications without requiring hardly any knowledge of the problem at
hand.  Also, the simplicity for which they can be applied implies that
practitioners can use them to find high quality solutions to a wide
variety of problems without needing skills and knowledge of algorithm
design. Indeed, numerous applications report high performance results
which make them widely used in practice.  However, through
experimental work and applications it is difficult to understand the
reasons for these successes.  In particular, given a stochastic search
algorithm, it is unclear on which kind of problems it will achieve
good \emph{performance} and on which it will perform poorly. Even more
crucial is the lack of understanding of how the parameter settings
influence the performance of the algorithms.  The goal of a rigorous
theoretical foundation of stochastic search algorithms is to answer
questions of this nature by \emph{explaining} the success or the
failure of these methods in practical applications. The benefits of a
theoretical understanding are threefold: (a) guiding the choice of the
best algorithm for the problem at hand, (b) determining the optimal
parameter settings, and (c) aiding the
algorithm design, ultimately leading to the achievement of better
algorithms.

Theoretical studies of stochastic optimisation methods have always
existed, albeit few, since these algorithms became popular.  In
particular, the increasing popularity gained by evolutionary
and genetic algorithms in
the seventies led to various attempts at building a theory for these
algorithms.  However, such initial studies attempted to provide
insights on the \emph{behaviour} of evolutionary algorithms rather than
estimating their performance.  The most popular of these theoretical
frameworks was probably the \emph{schema theory} introduced by Holland
\cite{Holland1975} and made popular by Goldberg \cite{Goldberg1989}.
In the early nineties a very different approach appeared to the
analysis of evolutionary algorithms and consequently randomised search
heuristics in general, driven by the insight that these heuristics are
indeed randomised algorithms, albeit general-purpose ones, and as such
they should be analysed in a similar spirit to that of classical
randomised algorithms \cite{MotwaniRaghavan1995}. For the last 25
years this field has kept growing considerably and nowadays several
advanced and powerful tools have been devised that allow the analysis
of the performance of involved stochastic search algorithms for
problems from various domains.  These include problems from classical
combinatorial and continuous optimisation, dynamic optimisation and
noisy optimisation.  The generality of the developed techniques, has
allowed their application to the analyses of several families of
stochastic search algorithms including evolutionary algorithms, local
search, metropolis, simulated annealing, ant colony optimisation,
artificial immune systems, particle swarm optimisation, estimation of
distribution algorithms amongst others.  

The aim of this chapter is to introduce the reader to the most common
and powerful tools used in the performance analysis of randomised
search heuristics.  Since the main focus is the understanding of the
methods, these will be applied to the analysis of very simple
evolutionary algorithms for artificial example functions. The hope is
that the structure of the functions and the behaviour of the algorithms
are easy to grasp so the attention of the reader may be mostly focused
on the mathematical techniques that will be presented. At the end of
the chapter references to complex applications of the techniques for
the obtainment of advanced results will be pointed out for further
reading.

\section{Computational Complexity of Stochastic Search Algorithms}

From the perspective of computer science, stochastic search heuristics
are randomised algorithms although more general than problem specific
ones.  Hence, it is natural to analyse their performance in the
classical way as done in computer science. From this perspective an
algorithm should be \emph{correct}, i.e., for every instance of the
problem (the input) the algorithm halts with the correct solution
(i.e., the correct output) and it should be efficient in terms of its
\emph{computational complexity} i.e., the algorithm uses the
computational resources wisely.  The resources usually considered are
the number of basic computations to find the solution (i.e., time) and
the amount of memory required (i.e., space).

Differently from problem-specific algorithms, the goal behind
general-purpose algorithms such as stochastic search heuristics is to
deliver good performance independently of the problem at hand.  In
other words, a general-purpose algorithm is ``correct'' if it
\emph{visits the optimal solution} of \emph{any problem} in finite
time.  If the optimum is never lost afterwards, then a stochastic
search algorithm is said to \emph{converge} to the optimal solution.
In a formal sense the latter condition for convergence is required
because most search heuristics are not capable of recognising when an
optimal solution has been found (i.e., they do not halt).  However, it
suffices to keep track of the best found solution during the run of
the algorithm, hence the condition is trivial to satisfy.  What is
particularly relevant and can make a huge difference on the usefulness
of a stochastic search heuristic for a given problem is its \emph{time
  complexity}.  In each iteration the evaluation of the quality of a
solution is generally far more expensive than its other algorithmic
steps.  As a result, it is very common to measure \emph{time} as the
number of evaluations of the fitness function (also called objective
function) rather than counting the number of basic computations.
Since randomised algorithms make random choices during their
execution, the runtime of a stochastic search heuristic $A$ to
optimise a function $f$ is a random variable $T_{A,f}$.  The main
measure of interest is:
\begin{enumerate}
\item The \emph{expected runtime} $\expect{T_{A,f}}$: the expected
  number of fitness function evaluations until the optimum of $f$ is
  found; \\ 
  For skewed runtime distributions, the expected runtime may be a
  deceiving measure of algorithm performance. The
  following measure therefore provides additional information.
\item The \emph{success probability in $t$ steps} $\Prob{T_{A,f} \leq
    t}$: the probability that the optimum is found within $t$ steps.
\end{enumerate}
Just like in the classical theory of efficient algorithms the time is
analysed in relation to growing input length and usually described
using asymptotic notation \cite{CormenLeisersonRivestStein2001}. A
search heuristic $A$ is said to be \emph{efficient} for a function
(class) $f$ if the runtime grows as polynomial function of the
instance size. On the other hand, if the runtime grows as an
exponential function, then the heuristic is said to be
\emph{inefficient}. See Figure~\ref{fig:complexity} for an
illustrative distinction.

\begin{figure}[t]
\begin{center}
        \includegraphics[width=10cm]{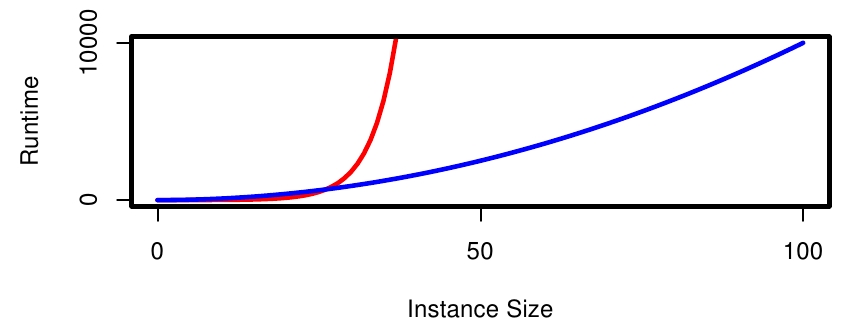}
      \end{center}
      \caption{An efficient search heuristic (blue) versus and
        inefficient search heuristic (red) for a given instance.}
\label{fig:complexity}   
\end{figure}

\section{Evolutionary Algorithms}
A general framework of an evolutionary algorithm is the
($\mu$+$\lambda$)~EA defined in Algorithm \ref{alg:mulambdaea}.  The
algorithm evolves a population of $\mu$ candidate solutions, generally
called the \emph{parent population}.  At each generation an
\emph{offspring population} of $\lambda$ individuals is created by
selecting individuals from the parent population uniformly at random
and by applying a mutation operator to them. The generation is
concluded by selecting the $\mu$ fittest individuals out of the $\mu +
\lambda$ parents and offspring. Algorithm \ref{alg:mulambdaea}
presents a formal definition.

\begin{algorithm}[h]
\caption{($\mu$+$\lambda$)~EA} \label{alg:mulambdaea}
 \begin{algorithmic}[1]
      \State  {\bf Initialisation}: 
      \Statex Initialise $P_0 = \{x^{(1)}, \dots, x^{(\mu)}\}$ with
              $\mu$ individuals chosen uniformly a random from $\{0,1\}^n$;
      \Statex $t\gets 0$; 
      \For{$i=1, \dots, \lambda$}  \label{lineFor}
	  \State {\bf Selection for Reproduction}: Choose $x\in P_t$ uniformly at random;
	  \State {\bf Variation}: Create $y^{(i)}$ by flipping each bit in $x$ with probability $p_m$;
      \EndFor
      \State {\bf Selection for Replacement}\label{lineSelection}
      \Statex Create the new population $P_{t+1}$ by choosing the best $\mu$ individuals out of 
      $\{x^{(1)}, \dots, x^{(\mu)}, y^{(1)}, \dots, y^{(\lambda)}\}$;
      \State $t\gets t+1$; Continue at \ref{lineFor};
    \end{algorithmic}
\end{algorithm}
In order to apply the algorithm for the optimisation of a fitness
function $f : \{0,1\}^n \rightarrow \mathbb{R}$, some parameters need
to be set. The population size $\mu$, the offspring population size
$\lambda$ and the mutation rate $p_m$.  Generally $p_m = 1/n$ is
considered a good setting for the mutation rate.  Also, in practical
applications a stopping criterion has to be defined since the
algorithm does not halt.  A fixed number of generations or a fixed
number of fitness function evaluations are usually decided in advance.
Since the objective of the analysis is to calculate the time required
to reach the optimal (approximate) solution for the first time, no
stopping condition is required, and one can assume that the algorithms
are allowed to run forever.  The + symbol in the algorithm's name
indicates that \emph{elitist truncation selection} is applied.  This
means that the whole population consisting of both parents and
offspring are sorted according to fitness and the best $\mu$ are
retained for the next generation.  Some criterion needs to be decided
in case the best $\mu$ individuals are not uniquely defined.  Ties
between solutions of equal fitness may be broken uniformly at
random. Often offspring are preferred over parents of equal fitness.
In the latter case if $\mu=\lambda=1$ are set, then the standard
(1+1)~EA is obtained, a very simple and well studied evolutionary
algorithm.  On the other hand if some stochastic selection mechanism
was used instead of the elitist mechanism and a crossover operator was
added as variation mechanism, then Algorithm \ref{alg:mulambdaea}
would become a genetic algorithm (GA) \cite{Goldberg1989}. Given the
importance of the (1+1)~EA in this chapter, a formal definition is
given in Algorithm \ref{alg:oneoneea}.
\begin{algorithm}[h]
\caption{(1+1)~EA} \label{alg:oneoneea}
 \begin{algorithmic}[1]
      \State  {\bf Initialisation}: 
      \Statex Initialise $x^{(0)}$ uniformly a random from $\{0,1\}^n$;
      \Statex $t\gets 0$;
      \State {\bf Variation}: 
      \Statex Create $y$ by flipping each bit in $x^{(t)}$ with probability $p_m = 1/n$; \label{variation}
      \State {\bf Selection for Replacement}
      \If{$f(y) \geq f(x^{(t)})$}
      \State $x^{(t+1)}\gets y$
      \EndIf
      \State $t\gets t+1$; Continue at \ref{variation};
    \end{algorithmic}
\end{algorithm}

The algorithm is initialised with a random bitstring. At each
generation a new candidate solution is obtained by flipping each bit
with probability $p_m = 1/n$.  The number of bits that flip can be
represented by a binomial random variable $X\sim\bin(n,p)$ where $n$
is the number of bits (i.e., the number of trials) and $p=1/n$ is the
probability of a success (i.e. a bit actually flips), while $1-p = 1
-1/n$ is the probability of a failure (i.e., the bit does not flip).
Then, the expected number of bits that flip in one generation is given
by the expectation of the binomial random variable, $\expect{X} = n
\cdot p = n \cdot 1/n = 1$.

The algorithm behaves in a very different way compared to the random
local search (RLS) algorithm that flips \emph{exactly} one bit per
iteration.  Although the (1+1)~EA flips exactly one bit in expectation
per iteration, many more bits may flip or even none at all.  In
particular, the (1+1)~EA is a \emph{global optimiser} because there is
a positive probability that any point in the search space is reached
in each generation.  As a consequence, the algorithm will find the
global optimum in finite time. On the other hand, RLS is a local
optimiser since it gets stuck once it reaches a local optimum because
it only flips one bit per iteration.

The probability that a binomial random variable $X\sim\bin(n,p)$ takes
value $j$ (i.e., $j$ bits flip) is
\[ \Prob{X=j} = {n \choose j} p^j (1 - p)^{n-j}.\] 
Hence, the probability that the (1+1)~EA flips exactly one bit is 
\[
 \Prob{X=1} = {n \choose 1} \cdot \left( \frac{1}{n}\right) \cdot
              \left(1 - \frac{1}{n}\right)^{n-1} 
            = \left( 1 - \frac{1}{n}\right)^{n-1} \geq 1/e \approx 0.37
\]
So the outcome of one generation of the (1+1)~EA is similar to that of
RLS only approximately 1/3 of the generations.  The probability that
two bits flip is exactly half the probability that one flips:
\begin{align*}
\Prob{X=2} 
   & = {n \choose 2}\left(\frac{1}{n}\right)^2 \left(1 - \frac{1}{n}\right)^{n-2} \\
   & = \frac{n (n-1)}{2}\left(\frac{1}{n}\right)^2 \left(1 - \frac{1}{n}\right)^{n-2} \\
  & = \frac{1}{2} \left(1 - \frac{1}{n}\right)^{n-1} \approx 1/(2e)   
\end{align*}
On the other hand the probability no bits flip at all is: 
\[
 \Prob{X=0} = {n \choose 0} (1/n)^0 \cdot (1-1/n)^n \approx 1/e
\]
The latter result implies that in more than 1/3 of the iterations no
bits flip.  This should be taken into account when evaluating the
fitness of the offspring, especially for expensive fitness functions.

In general, the probability that $i$ bits flip decreases exponentially
with $i$:
\[
 \Prob{X=i} = {n \choose i} \cdot \frac{1}{n^i} \cdot \left(1 - \frac{1}{n}\right)^{n-i} = \frac{1}{i!} \cdot \left( 1 - \frac{1}{n}\right)^{n-i} \approx \frac{1}{i! \cdot e}
 \]

 In the worst case all the bits may need to flip to reach the optimum
 in one step. This event has probability $1/n^n$.  Since, this is
 always a lower bound on the probability of reaching the optimum in
 each generation, by a simple waiting time argument an upper bound of
 $O(n^n)$ may be derived for the expected runtime of the (1+1)~EA on
 any pseudo-Boolean function $f : \{0,1\}^n \rightarrow \mathbb{R}$.
 It is simple to design an example trap function for which the
 algorithm actually requires $\Theta(n^n)$ expected steps to reach the
 optimum \cite{DrosteJansenWegener2002}.  This simple result further
 motivates why it is fundamental to gain a foundational understanding
 of how the runtime of stochastic search heuristics depends on the
 parameters of the problem and on the parameters of the algorithms.

\section{Test Functions}
\begin{figure}[t]
\begin{center}
        \includegraphics[width=7cm]{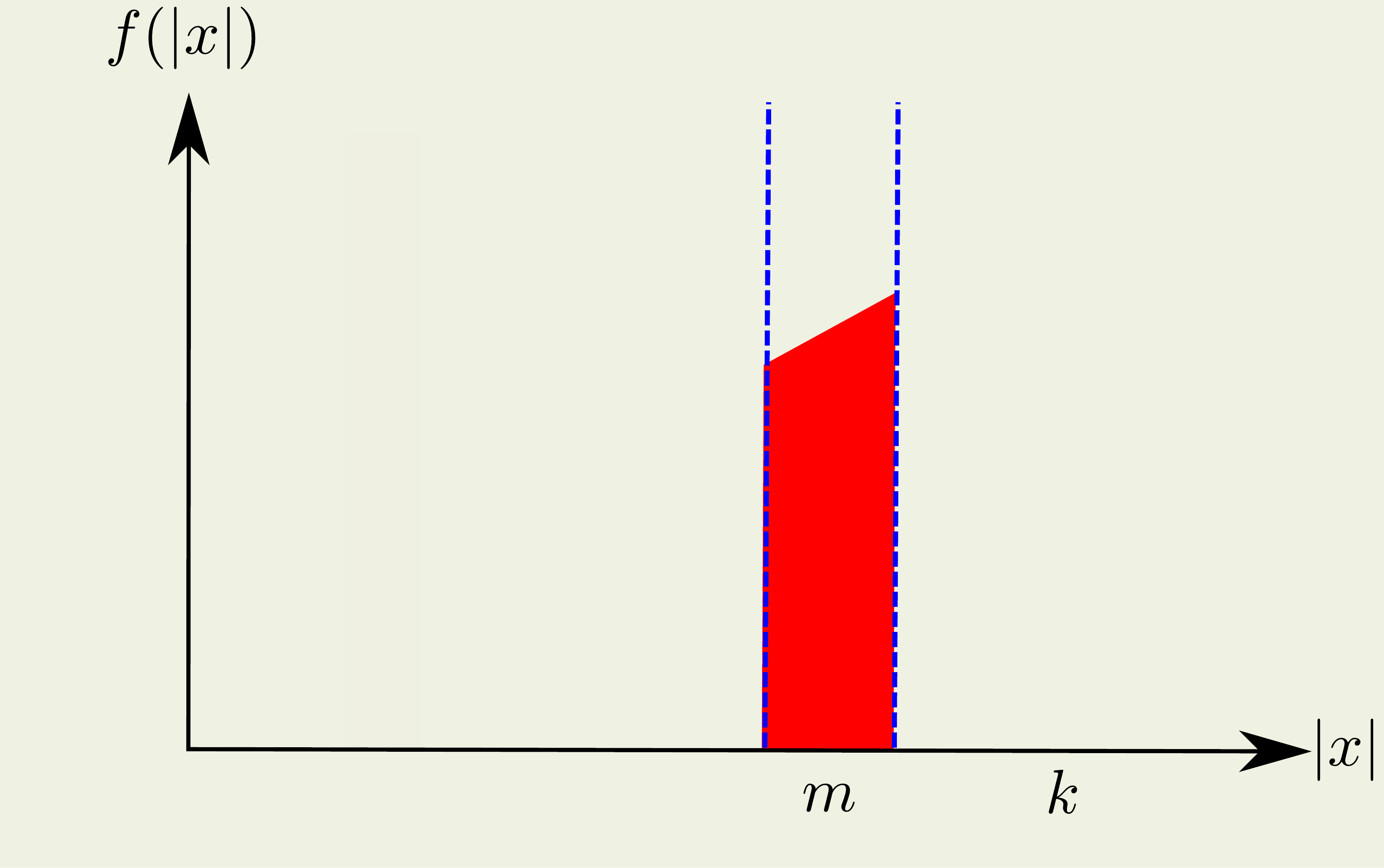}
      \end{center}
      \caption{A linear unitation block of length $m$ starting at
        position $m+k$ and ending at position $k$. A linear unitation
        block of length $n$ is the \ONEMAX function.}
\label{fig:linearblock}       
\end{figure}

Test functions are artificially designed to analyse the performance of
stochastic search algorithms when they face optimisation problems with
particular characteristics.  These functions are used to highlight
characteristics of function classes which may make the optimisation
process easy or hard for a given algorithm.  For this reason they are
often referred to as \emph{toy problems}. The analysis on test
functions of simple and well understood structure has allowed the
development of several general techniques for the analysis. Afterwards
these techniques have allowed to analyse the same algorithms for more
complicated problems with practical applications such as classical
combinatorial optimisation problems. Furthermore, in recent years
several standard techniques originally developed for simple algorithms
have been extended to allow the analyses of more realistic
algorithms. In this section the test functions that will be used as
example functions throughout the chapter are introduced.

The most popular test function is \ONEMAX(x) := $\sum_{i=1}^n x_i$
which simply counts the number of one-bits in the bitstring. The
global optimum is a bitstring of only one-bits.  \ONEMAX is the
easiest function with unique global optimum for the (1+1)~EA
\cite{EasiestGecco15}.

A particularly difficult test function for stochastic search
algorithms is the \emph{needle-in-a-haystack} function.  $\NEEDLE(x)
:= \prod_{i=1}^n x_i$ consists of a huge plateau of fitness value zero
apart from only one optimal point of fitness value one represented by
the bitstring of only one-bits.  This function is hard for search
heuristics because all the search points apart from the optimum have
the same fitness.  As a consequence, the algorithms cannot gather any
information about where the needle is by sampling search points.

Both \ONEMAX and \NEEDLE (as defined above) have the property that the
function values only depend on the number of ones in the bitstring.
The class of functions with this property is called \emph{functions of
  unitation}
\[
 \UNITATION(x) := f \left(\sum_{i=1}^n x_i\right)
\]
Throughout this chapter, functions of unitation will be used as a
general example class to demonstrate the use of the techniques that
will be introduced.  For simplicity of the analysis, the optimum is
assumed to be the bitstring of only one-bits.

For the analysis the function of unitation will be divided into three
different kinds of sub-blocks: linear blocks, gap blocks and plateau
blocks.  Each block will be defined by its length parameter $m$
(i.e. the number of bits in the block) and by its position parameter
$k$ (i.e., each block starts at bitstrings with $m+k$ zeroes and ends
at bitstrings with $k$ zeroes).  Given a unitation function it is
divided into sub-blocks proceeding from left to right from the
all-zeroes bitstring towards the all-ones bitstring.  If the fitness
increases with the number of ones, then a \emph{linear block} is
created. The linear block ends when the function value stops
increasing with the number of ones.
\[
\LINEAR (|x|)  =  \left\{ \begin{array}{ll}
    a|x| + b & \text{ if } k < n-|x| \leq k+m\\
           0 & \text{ otherwise.}
  \end{array} \right.       
\]
See Figure~\ref{fig:linearblock} for an illustration.

If the fitness function decreases with the number of ones, then a
\emph{gap block} is created. The gap block ends when the fitness value
reaches for the first time a higher value than the value at the
beginning of the block.
\[
\GAP (|x|)  =  \left\{ \begin{array}{ll}
    a & \text{ if }  n-|x| = k+m\\
           0 & \text{ otherwise.}
  \end{array} \right.       
\]
See Figure~\ref{fig:gapblock} for an illustration.

\begin{figure}[t]
\begin{center}
        \includegraphics[width=7cm]{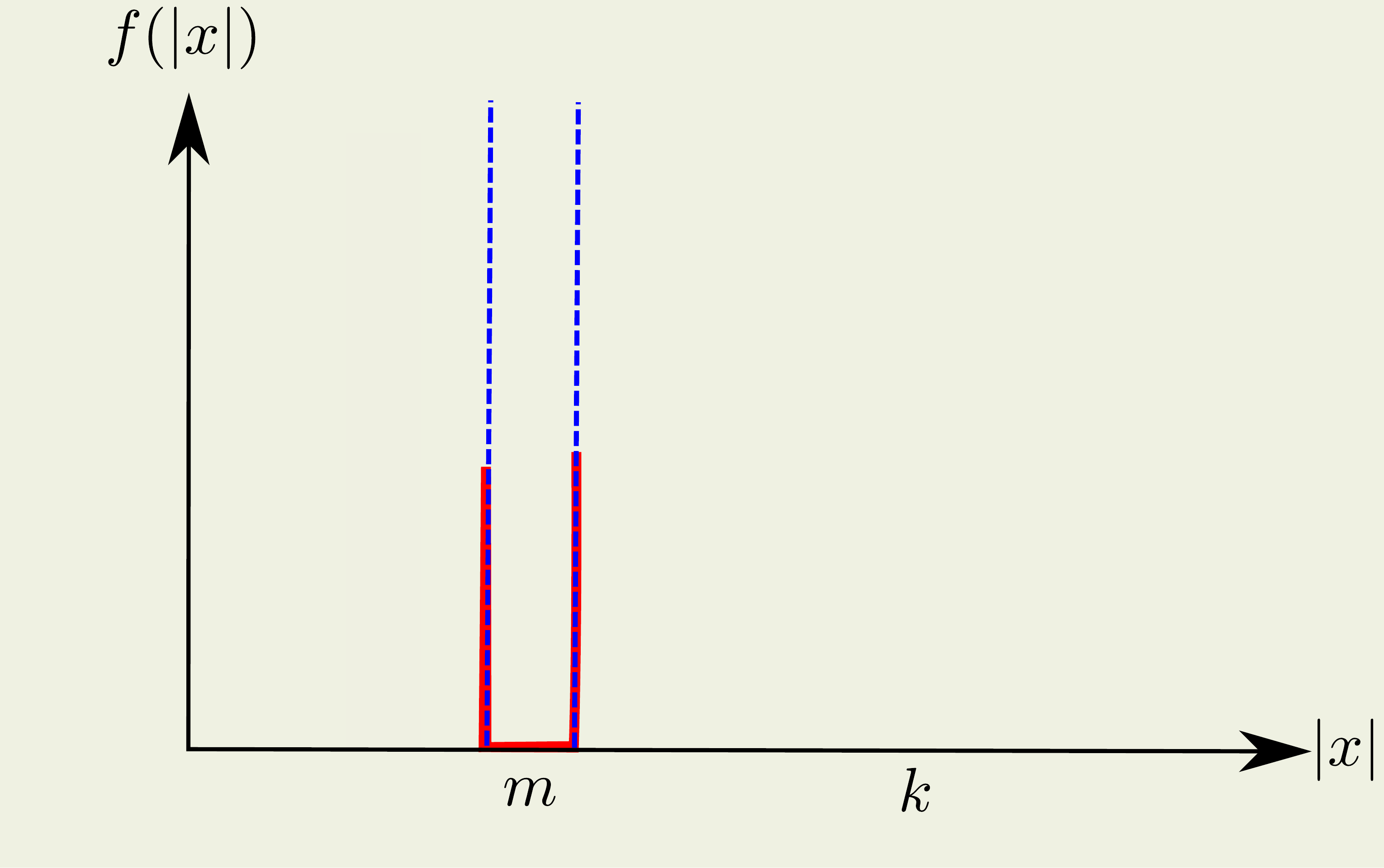}
      \end{center}
      \caption{A gap unitation block of length $m$ starting at
        position $m+k$ and ending at position $k$. A gap unitation
        block of length $n-1$ is the \NEEDLE function.}
\label{fig:gapblock}       
\end{figure}

If the fitness remains the same as the number of ones in the
bitstrings increases, then a \emph{plateau block} is created. The
block ends at the first point where the fitness value changes.
\[
\PLATEAU (|x|)  =  \left\{ \begin{array}{ll}
    a & \text{ if } k<  n-|x| \leq k+m\\
           0 & \text{ otherwise.}
  \end{array} \right.       
\]
See Figure~\ref{fig:plateaublock} for an illustration.

\begin{figure}[t]
\begin{center}
        \includegraphics[width=7cm]{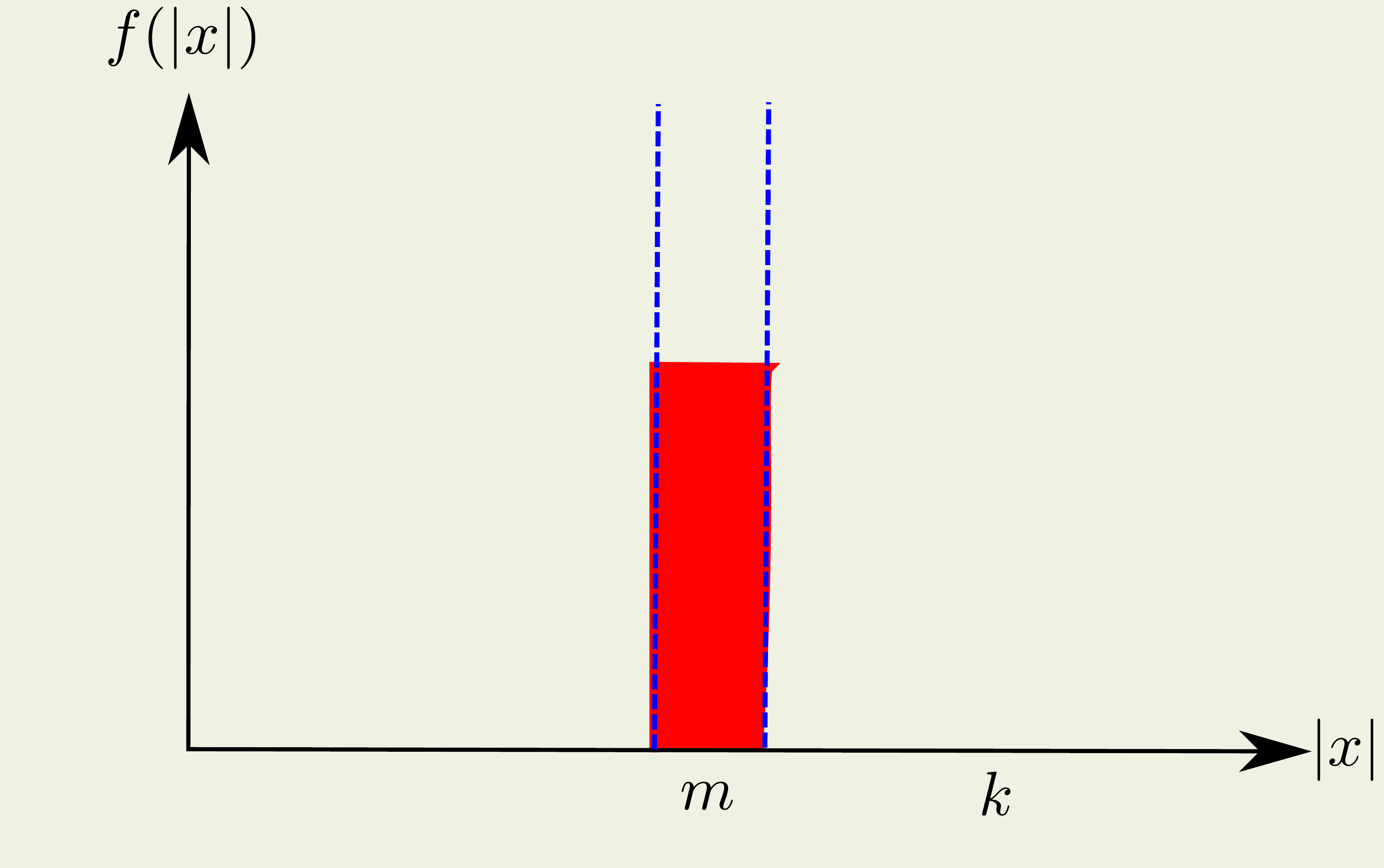}
      \end{center}
      \caption{A plateau unitation block of length $m$ starting at
        position $m+k$ and ending at position $k$.}
\label{fig:plateaublock}       
\end{figure}

By proceeding from left to right the whole search space is subdivided
into blocks. See Figure~\ref{gif:unitationblocks} for an illustration.
\begin{figure}
\centering
\begin{minipage}{.3\textwidth}
  \centering
  \includegraphics[width=1.0\linewidth]{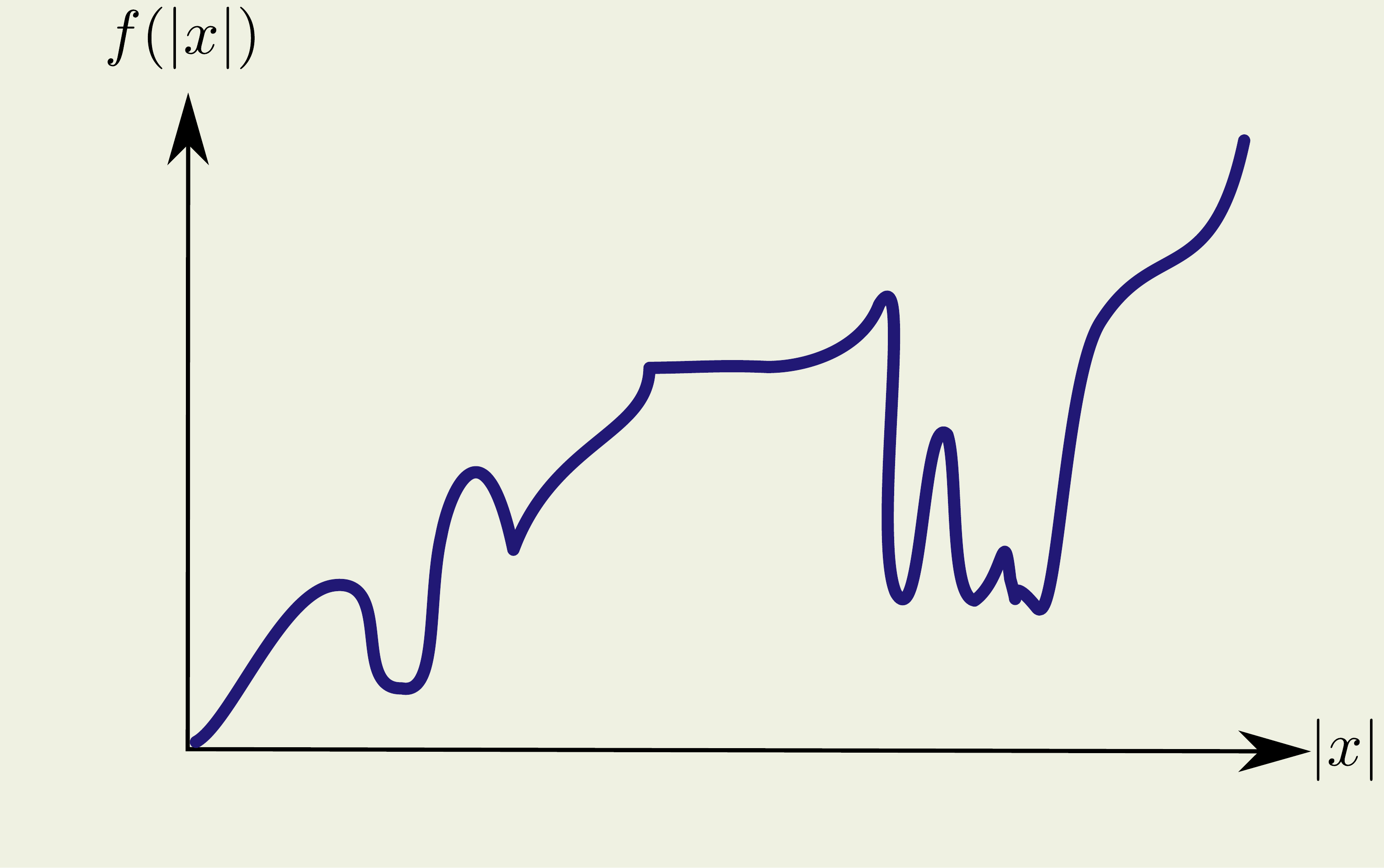}
\end{minipage}%
\begin{minipage}{.3\textwidth}
  \centering
  \includegraphics[width=1.0\linewidth]{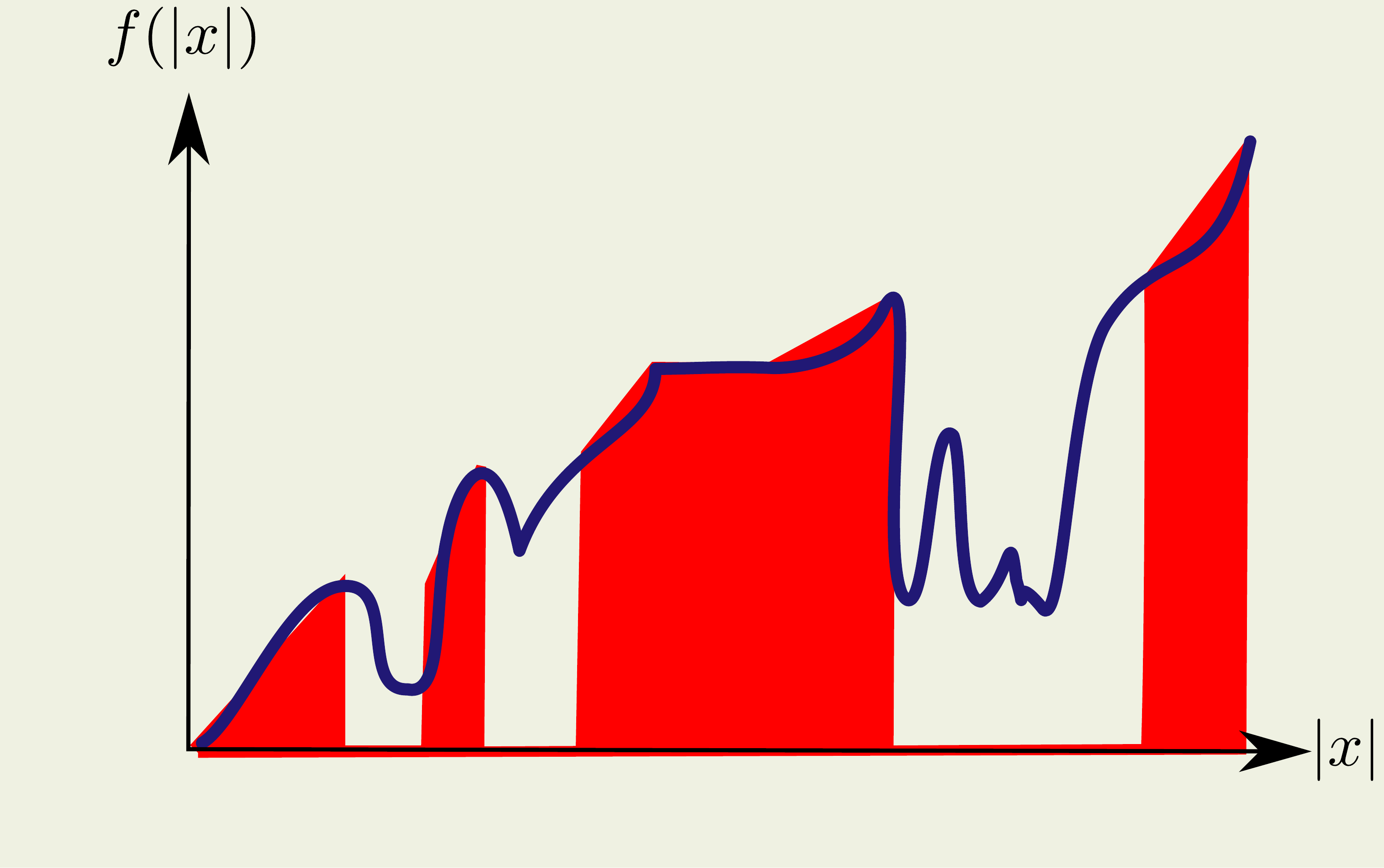}
\end{minipage}
\begin{minipage}{.3\textwidth}
  \centering
  \includegraphics[width=1.0\linewidth]{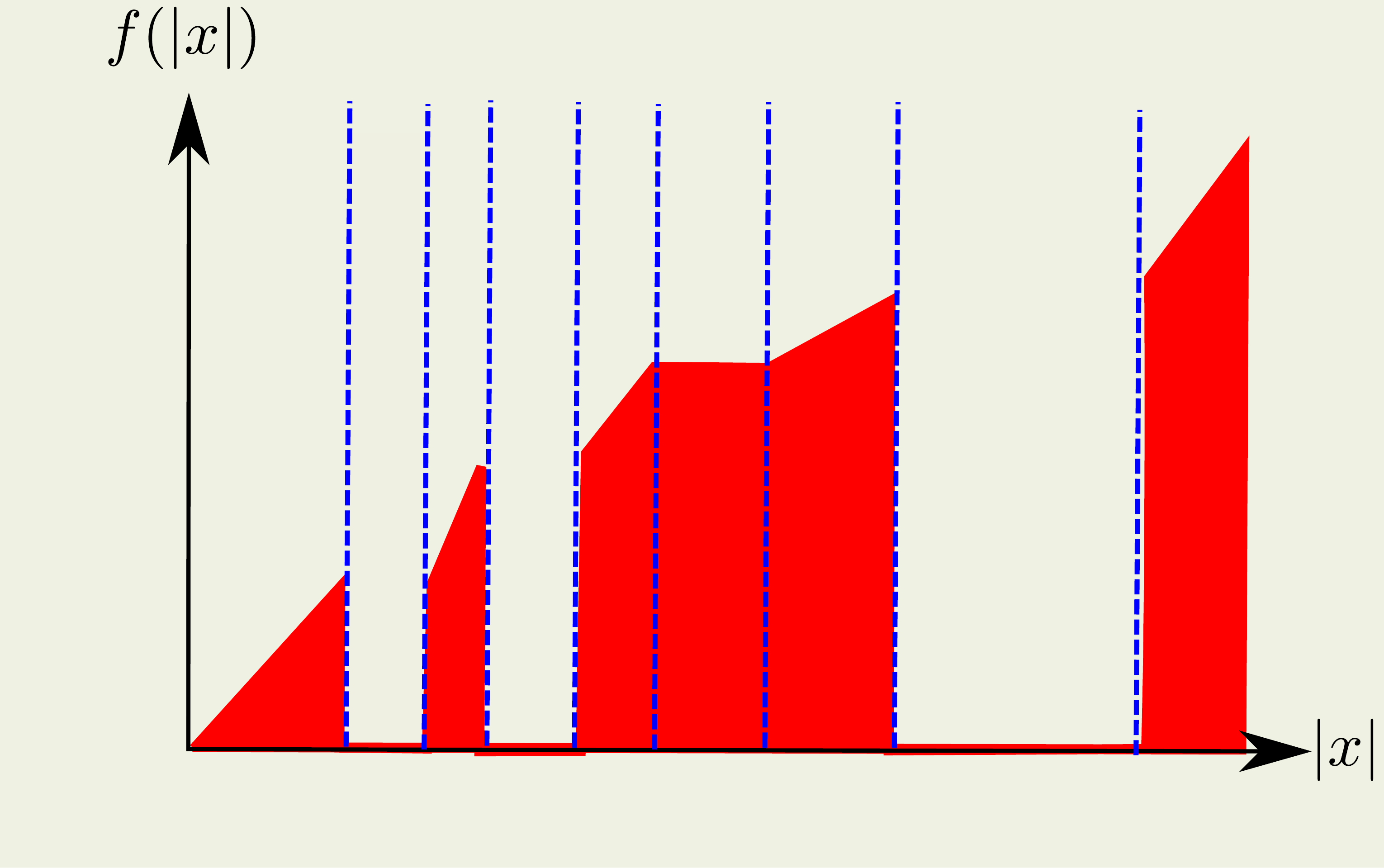}
\end{minipage}
\caption{An illustration of how the search space of a unitation
  function may be subdivided into blocks of the three kinds.}
\label{gif:unitationblocks}
\end{figure}
Let the unitation function be subdivided into $r$ sub-functions $f_1,
f_2, \dots f_r$, and let $T_i$ be the runtime for an elitist search
heuristic to optimise each sub-function $f_i$. Then by linearity of
expectation, an upper bound on the expected runtime of an elitist
stochastic search heuristic for the unitation function is:
 \[
  \expect{T} \leq \expect{\sum_{i=1}^r T_i} = \sum_{i=1}^r \expect{T_i} .
  \]

  Hence, an upper bound on the total runtime for the unitation
  function may be achieved by calculating upper bounds on the runtime
  for each block separately.  Once these are obtained, summing all the
  bounds yields an upper bound on the total runtime.  Attention needs
  to be put when calculating upper bounds on the runtime to overcome a
  plateau block when this is followed by a gap block because points
  straight after the end of the plateau will have lower fitness
  values, hence will not be accepted. In these special cases, the
  upper bound for the \PLATEAU block needs to be multiplied by the
  upper bound for the \GAP block to achieve a correct upper bound on
  the runtime to overcome both blocks.  In the remainder of the
  chapter upper and lower bounds for each type of block will be
  derived as example applications of the presented runtime analysis
  techniques.  The reader, will then be able to calculate the runtime
  of the (1+1)~EA and other evolutionary algorithms for any such
  unitation function.

  By simply using waiting time arguments it is possible to derive
  upper and lower bounds on the runtime of the (1+1)~EA for the \GAP
  block.  Assuming that the algorithm is at the beginning of the gap
  block then to reach the end it is \emph{sufficient} to flip $m$
  zero-bits into one-bits and leave the other bits unchanged.  On the
  other hand it is a \emph{necessary} condition to flip at least $m$
  zero-bits because all search points achieved by flipping less than
  $m$ zero-bits have a fitness value of zero and would not be accepted
  by selection. Given that there are $m+k$ zero-bits available at the
  beginning of the block, the following upper and lower bounds on the
  probability of reaching the end of the block follows
\[
 \left(\frac{m+k}{nm}\right)^m \frac{1}{e} \leq 
{m+k \choose m} \left( \frac{1}{n}\right)^m \frac{1}{e}
\leq {p} \leq {m+k \choose m} \left( \frac{1}{n}\right)^m 
{\leq \left(\frac{(m+k)e}{nm}\right)^m}.
\]
Here the outer inequalities are achieved by using
$\left(\frac{n}{k}\right)^k \leq {n\choose k}\leq
\left(\frac{en}{k}\right)^k$ for $k\geq 1$.  Then by simple waiting
time arguments, the expected time for the (1+1)~EA to optimise a \GAP
block of length $m$ and position $k$ is upper and lower bounded by
\[
 {\left( \frac{nm}{(m+k)e}\right)^m \leq }
{m+k \choose m}^{-1} n^m \leq {\expect{T}}
\leq en^m  {m+k \choose m}^{-1} {\leq  e  \left( \frac{nm}{m+k}\right)^m }.
\]

\section{Tail Inequalities}

%
%

\begin{figure}[t]
\begin{center}
        \includegraphics[width=9cm]{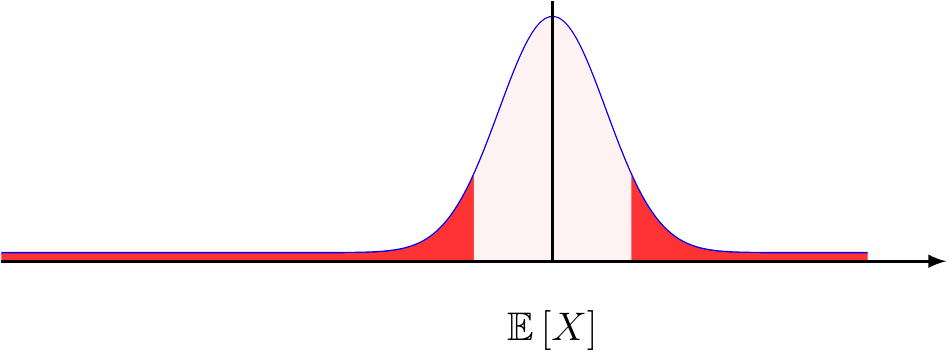}
      \end{center}
      \caption{The expectation of a random variable and its probability
  distribution. The tails are highlighted in red.}
\label{fig:distribution0}       
\end{figure}

The runtime of a stochastic search algorithm $A$ for a function
(class) $f$ is a random variable $T_{A,f}$ and the main goal of a
runtime analysis is to calculate its expectation $\expect{T_{A,f}}$.
Sometimes the expected runtime may be particularly large, but there
may also be a high probability that the actual optimisation time is
significantly lower.  In these cases a result about the \emph{success
  probability} within $t$ steps, helps considerably the understanding
of the algorithm's performance. In other occasions it may be
interesting to simply gain knowledge about the probability that the
actual optimisation time deviates from the expected runtime. In such
circumstances \emph{tail inequalities} turn out to be very useful
tools by allowing to obtain bounds on the runtime that hold with high
probability. An example of the expectation of a random variable and
its probability distribution are given in
Figure~\ref{fig:distribution0}.

Given the expectation of a random variable, which often may be
estimated easily, tail inequalities give bounds on the probability
that the actual random variable deviates from its expectation
\cite{MotwaniRaghavan1995,MitzenmacherUpfal2005}. The most simple tail
inequality is Markov's inequality. Many strong tail inequalities are
derived from Markov's inequality.

\begin{theorem}[Markov's Inequality]
  Let $X$ be a random variable assuming only non-negative values. Then for all $t \in \mathbb{R}^+$,
 \begin{align*}
   \Pr(X \ge t) \le \frac {\expect{X}} t.
 \end{align*}
\end{theorem}

The power of the inequality is that no knowledge about the random
variable is required apart from it being non-negative.

Let $X$ be a random variable indicating the number of bits flipped in
one iteration of the (1+1)~EA.  As seen in the previous section, one
bit is flipped per iteration in expectation, i.e., $\expect{X}=1$.
One may wonder what is the probability that more than one bit is
flipped in one time step.  A straightforward application of Markov's
Inequality reveals that in at least half of the iterations either one
bit is flipped or none:
\[
 \Prob{X \geq 2} \leq \frac{\expect{X}}{2} = \frac{1}{2}
\]
Similarly, one may want to gain some information on how many
\emph{ones} are contained in the bitstring at initialisation, given
that in expectation there are $\expect{X} = n/2$ (here $X$ is a
binomial random variable with parameters $n$ and $p=1/2$).  An
application of Markov's inequality yields that the probability of
having more than $(2/3)n$ ones at initialisation is bounded by
\begin{align}\label{eq:1}
  \Prob{X \geq (2/3)n} \leq \frac{\expect{X}}{(2/3)n} = \frac{n/2}{(2/3)n} = 3/4
\end{align}
Since $X$ is binomially distributed it is reasonable to expect that,
for large enough $n$, the actual number of obtained ones at
initialisation would be more concentrated around the expected
value. In particular while the bound is obviously correct, the
probability that the initial bitstring has more than $(2/3)n$ ones is
much smaller than $3/4$.  However, to achieve such a result more
information about the random variable should be required by the tail
inequality (i.e., that it is binomially distributed).  An important
class of tail inequalities used in the analysis of stochastic search
heuristics are Chernoff bounds.

\begin{theorem}[Chernoff Bounds]
Let $X_1,X_2, \dots X_n$  be independent random variables taking values in $\{0,1\}$.
Define $X= \sum_{i=1}^n X_i$, which has expectation $E(X) = \sum_{i=1}^n \Pr(X_i=1)$.
    \begin{itemize}
    \item[(a)]   $\Pr(X \leq (1- \delta)\expect{X}) \le e^{\frac{-\expect{X}\delta^2}{2}}$
      for $0 \le \delta \le 1$.
    \item[(b)]
      $\Pr(X > (1+ \delta)\expect{X}) \le \left( \frac {e^\delta}
        {(1+\delta)^{1+\delta}} \right)^{\expect{X}}$  for $\delta  > 0$.
    \end{itemize}
  \end{theorem}
  
  An application of Chernoff bounds reveals that the probability that
  the initial bitstring has more than $(2/3)n$ one-bits is
  exponentially small in the length of the bitstring.  Let $X=
  \sum_{i=1}^n X_i$ be the random variable summing up the random
  values $X_i$ of each of the $n$ bits.  Since each bit is initialised
  with probability $1/2$, it holds that $\Pr(X_i=1)=1/2$ and
  $\expect{X} = n/2$.  By fixing $\delta = 1/3$ it follows that $(1+
  \delta)\expect{X}= (2/3)n$ and finally by applying inequality (b),
\[
\Pr(X > (2/3)n) \le \left(\frac{e^{1/3}}{(4/3)^{4/3}}\right)^{n/2} < \left(\frac{29}{30}\right)^{n/2} 
\]
In fact an exponentially small probability of deviating from $n/2$ by
a constant factor of the search space $c/n$ for any constant $c>0$ may
easily be obtained by Chernoff bounds.

\section{Artificial Fitness Levels (AFL)} \label{sec:fitnesslevels}

The artificial fitness levels technique is a very simple method to
achieve upper bounds on the runtime of elitist stochastic optimisation
algorithms.  Albeit its simplicity, it often achieves very good bounds
on the runtime.

The idea behind the method is to divide the search space of size $2^n$
into $m$ disjoint fitness-based partitions $A_1, \dots A_m$ of
increasing fitness such that $f(A_i) < f(A_j)$ $\forall i < j$.  The
union of these partitions should cover the whole search space and the
level of highest fitness $A_m$ should contain the global optimum (or
all global optima if there is more than one).

\begin{definition}\label{def:f-based-partition}
  A tuple $(A_1,A_2,\ldots,A_m)$ is an {\bf $f$-based partition} of $f:\mathcal{X}\rightarrow\mathbb{R}$ if
         \begin{enumerate}
         \item $A_1\cup A_2\cup \cdots\cup A_m = \mathcal{X}$
         \item $A_i\cap A_j=\emptyset$ for $i\neq j$
         \item $f(A_1)< f(A_2)< \cdots < f(A_m)$
         \item $f(A_m) = \max_x f(x)$
         \end{enumerate}    
\end{definition}

  \begin{figure}[t]
\centering

\begin{tikzpicture}[scale=0.9,transform shape]

  \draw[->] (-0.25,0) -- node[left] {Fitness} (-0.25,6);
  \clip[draw] (2,0) .. controls (0,3) .. (2,6) .. controls (4,3) .. (2,0);
  \draw[fill=blue!20] (2,0) .. controls (0,3) .. (2,6) .. controls (4,3) .. (2,0);

    \draw[fill=blue!20] (0.0,0.0) rectangle (4.0,1.0);
    \draw[fill=blue!30] (0.0,1.0) rectangle (4.0,2.0);
    \draw[fill=blue!20] (0.0,2.0) rectangle (4.0,3.0);
    \draw[fill=blue!30] (0.0,3.0) rectangle (4.0,4.0);
    \draw[fill=blue!20] (0.0,4.0) rectangle (4.0,5.0);
    \draw[fill=blue!30] (0.0,5.0) rectangle (4.0,6.0);

    \node at (2,0.7) {$A_1$};
    \node at (2,1.5) {$A_2$};
    \node at (2,2.5) {$A_3$};
    \node at (2,3.5) {$\vdots$};
    \node at (2,4.5) {$A_{m-1}$};
    \node at (2,5.3) {$A_m$};
    \draw (2,0) .. controls (0,3) .. (2,6) .. controls (4,3) .. (2,0);
 
\end{tikzpicture}

\caption{A partition of the search space satisfying the conditions of an $f$-based partition.}
\label{fig:fitnesslevels}
\end{figure}
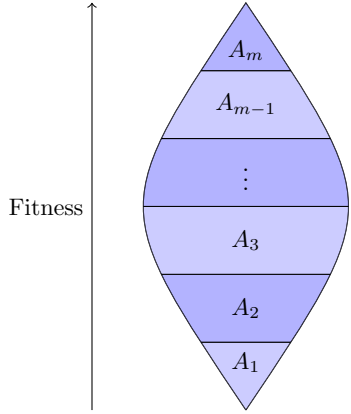

For functions of unitation, a natural way of defining a fitness-based
partition is to divide the search space into $n+1$ levels, each
defined by the number of ones in the bitstring. For the \ONEMAX
function, where fitness increases with the number of ones in the
bitstring, the fitness levels would be naturally defined as $A_i :=
\left\{ x\in \{0,1\}^n \mid \ONEMAX(x)=i \right\}$.

\subsection{AFL - Upper Bounds} 
 
Given a fitness-based partition of the search space, it is obvious
that an elitist algorithm using only one individual will only accept
points of the search space that belong to levels of higher or equal
fitness to the current level. Once a new fitness level has been
reached, the algorithm will never return to previous levels.  This
implies that each fitness level has to be left at most once by the
algorithm. Since in the worst case all fitness levels are visited, the
sum of the expected times to leave all levels is an upper bound on the
expected time to reach the global optimum.  The artificial fitness
levels method simplifies this idea by only requiring a lower bound
$s_i$ on the probability of leaving each level $A_i$ rather than
asking for the exact probabilities to leave each level.

\begin{theorem}[Artificial Fitness Levels] \label{AFL} Let
  $f:\mathcal{X}\rightarrow\mathbb{R}$ be a fitness function, $A_1
  \dots A_{m}$ a fitness-based partition of $f$ and $s_1 \dots
  s_{m-1}$ be lower bounds on the corresponding probabilities of
  leaving the respective fitness levels for a level of better
  fitness. Then the expected runtime of an elitist algorithm using a
  single individual is $\expect{T_{A,f}} \leq \sum_{i=1}^{m-1} 1/s_i$.
\end{theorem}

The artificial fitness level method will now be applied to derive an
upper bound on the expected runtime of (1+1)~EA for the \ONEMAX
function.  Afterwards, the bound will be generalised to general linear
blocks of unitation.

\begin{theorem} \label{th:onemax}
 The expected runtime of the (1+1)~EA on \ONEMAX is $O(n \ln n)$.
\end{theorem}

\begin{proof}
  The artificial fitness levels method will be applied to the $n+1$
  partitions defined by the number of ones in the bitstring, i.e.,
  $A_i := \left\{ x\in \{0,1\}^n \mid \ONEMAX(x)=i \right\}$. This
  means that all bitstrings with $i$ ones and $n-i$ zeroes belong to
  fitness level $A_i$.  For each level $A_i$, the method requires a
  lower bound on the probability of reaching any level $A_j$ where
  $j>i$.  To reach a level of higher fitness it is necessary to
  increase the number of ones in the bitstring.  However, it is
  sufficient to flip a \emph{zero} into a \emph{one} and leave the
  remaining bits unchanged.  Since the probability of flipping a bit
  is $1/n$ and there are $n-i$ zeroes that may be flipped, a lower
  bound on the probability to reach a level of higher fitness from
  level $A_i$ is:
\[
 s_i \geq (n-i) \cdot \frac{1}{n} \cdot \bigg(1 - \frac{1}{n}\bigg)^{n-1} \geq \frac{n-i}{en}
\]
where $(1- 1/n)^{n-1}$ is the probability of leaving $n-1$ bits unchanged and the inequality follows because $(1-1/n)^{n-1} \geq 1/e$ for all $n \in \mathbb{N}$.

Then by the artificial fitness levels method (Theorem~\ref{AFL}),
\[
 \expect{T_{(1+1)~EA, \ONEMAX}} \leq \sum_{i=0}^{m-1} 1/s_i \leq \sum_{i=0}^{n-1} \frac{en}{n-i} = en \sum_{i=1}^{n} \frac{1}{i} = O(n \ln n).
\]
\end{proof}

\begin{theorem}
  The expected runtime of the (1+1)-EA for a \emph{linear block} of
  length $m$ ending at position $k$ is $O(n \ln ((m+k)/k))$.
\end{theorem}

\begin{proof}
  Apply the artificial fitness levels method where each partition
  $A_i$ consists of the bitstrings in the block with $i$ zeroes.  Then
  the probability of leaving a fitness level is bounded by $s_i \geq
  i/n \cdot (1-1/n)^{n-1} \geq i/en$.  Given that at most $m$ fitness
  levels need to be left and that the block starts at position $m+k$
  and ends at position $k$, by Theorem~\ref{AFL} the expected runtime
  is:
\[
 \expect{T} \leq \sum_{i=k+1}^{k+m} \frac{en}{i} \leq en \sum_{i=k+1}^{k+m} \frac{1} {i}
\leq e n  \left(\sum_{i=1}^{k+m} \frac{1} {i} - \sum_{i=1}^{k} \frac{1} {i} \right)  
\leq e n  \ln \left(\frac{m+k}{k}\right)
\]
\end{proof}

\subsection{AFL - Lower Bounds}

Recently Sudholt introduced an artificial fitness levels method to
obtain lower bounds on the runtime of stochastic search algorithms
\cite{Sudholt2010}.  Since lower bounds are aimed for, apart from the
probabilities of leaving each fitness level, the method needs to also
take into account the probability that some levels may be skipped by
the algorithm.

\begin{theorem} \label{AFL-LowerBounds} Consider a fitness function
  $f:\mathcal{X}\rightarrow\mathbb{R}$ and $A_1 \dots A_m$ a
  fitness-based partition of $f$.  Let $u_i$ be the probability of
  starting in level $A_i$, $s_i$ be an upper bound on the probability
  of leaving $A_i$ and $p_{i,j}$ be an upper bound on the probability
  of jumping from level $A_i$ to level $A_j$.  If there exists some
  $0< \chi \leq 1$ such that for all $j>i$
\[
  p_{i,j}  \geq \chi \cdot {\sum_{k=j}^{m-1} p_{i,k}},
\]
then the expected runtime of an elitist algorithm using a single individual is 
\[
 \expect{T_{A,f}} \geq \chi \cdot \sum_{i=1}^{m-1} u_i \sum_{j=i}^{m-1} \frac{1}{s_j}
\]
\end{theorem}

The method will first be illustrated for the (1+1)~EA on the \ONEMAX
function.  Afterwards, the result will be generalised to general
linear blocks of unitation.

\begin{theorem}\label{thm:flevels-onemax}
 The expected runtime of the (1+1)~EA on \ONEMAX is $\Omega(n \ln n)$.
\end{theorem}

\begin{proof}
  Apply the artificial fitness levels method on the $n+1$ partitions
  defined by the number of ones in the bitstring, i.e., $A_i :=
  \left\{ x\in \{0,1\}^n \mid \ONEMAX(x)=i \right\}$. This means that
  all bitstrings with $i$ ones and $n-i$ zeroes belong to fitness
  level $A_i$.  To apply the artificial fitness levels method, bounds
  on $s_i$ and $\chi$ need to be derived.  An upper bound on the
  probability of leaving fitness level $A_i$ is simply $s_i \leq
  (n-i)/n$ because it is a necessary condition that at least one zero
  flips to reach a better fitness level.  The bound follows because
  each bit flips with probability $1/n$ and there are $n-i$ zeroes
  available to be flipped.  In order to obtain an upper bound on
  $\chi$, the method requires a lower bound on $p_{i,j}$ and an upper
  bound on $\sum_{k=j}^{m-1} p_{i,k}$.  For the lower bound on
  $p_{i,j}$ notice that in order to reach level $A_j$, it sufficient
  to flip $j-i$ zeroes out of the $n-i$ zeroes available and leave all
  the other bits unchanged. Hence the following bound is obtained:
\[
  p_{ij}  \geq    {n-i\choose j-i}\left(\frac{1}{n}\right)^{j-i}\left(1-\frac{1}{n}\right)^{n-(j-i)}
\]
For an upper bound on the sum, notice that to reach any level $A_k$ $k
\geq j$ from level $A_i$ it is necessary to flip at least $j-i$ zeroes
out of the $n-i$ available zeroes.  So,
\[
 \sum_{k=j}^{n-1} p_{i,k} \leq {n-i\choose j-i}\left(\frac{1}{n}\right)^{j-i}
\]
and for $\chi := 1/e$ the condition of Theorem~\ref{AFL-LowerBounds} is satisfied as follows:
\[
p_{i,j} \geq \left(1-\frac{1}{n}\right)^{n-(j-i)} \cdot \sum_{k=j}^{n-1} p_{i,k} \geq \chi \cdot \sum_{k=j}^{n-1} p_{i,k}
\]

By Eq. (\ref{eq:1}), the probability that the initial search point has
less than $(2/3)n$ 1-bits is at least 
\begin{align*}
  \sum_{i=1}^{(2/3)n} u_i \geq 1-\frac{3}{4}
\end{align*}
The statement of Theorem~\ref{AFL-LowerBounds} now yields
\begin{align*}
  \expect{T_{A,f}} & \geq \left(\frac{1}{e}\right)\cdot\sum_{i=1}^{n-1}u_i\sum_{j=i}^{n-1}\frac{1}{s_j}\\
            & >     \left(\frac{1}{e}\right)\cdot\left(\sum_{i=1}^{(2/3)n}u_i\right)\sum_{j=(2/3)n}^{n-1}\frac{1}{s_j}\\
            & \geq \left(\frac{1}{e}\right)\cdot\left(1-\frac{3}{4}\right)\sum_{j=(2/3)n}^{n-1}\frac{n}{n-j}\\
            & \geq \left(\frac{n}{4e}\right)\cdot \sum_{j=1}^{n/3}\frac{1}{j}.
\end{align*}
It now follows that $\expect{T_{A,f}}=\Omega(n\log n)$.
\end{proof}
Similarly the following result may also be proved for linear blocks of
unitation functions by defining the fitness partitions as 
$A_i := \{ x : n-|x| = k + m - i \}$ for $0\leq i\leq m$.

\begin{theorem}
  The expected runtime of the (1+1)-EA for a \emph{linear block} of
  length $m$ ending at position $k$ is $\Omega(n \ln ((m+k)/k))$.
\end{theorem}

\subsection{Level-based analysis of non-elitist populations}\label{AFL-NonElitist}

A weakness with the classical artificial fitness level technique is
that it is limited to search heuristics that only keep one solution,
such as the (1+1) EA, and it heavily relies on the selection mechanism
to use elitism.  \cite{Lehre2014} recently introduced the so-called
\emph{level-based analysis}, a generalisation of fitness level
theorems for non-elitist evolutionary algorithms which is also
applicable to search heuristics with populations, and using higher
arity operators such as crossover.

Their theorem applies to any algorithm that can be expressed in the
form of Algorithm~\ref{alg:pop}, such as genetic algorithms
\cite{Lehre2014} and estimation of distribution algorithms UMDA
\cite{DangLehre2015UMDA}.  The main component of the algorithm is a
random operator $D$ which given the current population
$P_t\in\mathcal{X}^\lambda$ returns a probability distribution
$D(P_t)$ over the search space $\mathcal{X}$. The next population
$P_{t+1}$ is obtained by sampling individuals independently from this
distribution.

\begin{algorithm}
  \caption{\label{alg:pop}Population-based algorithm with independent sampling}
  \begin{algorithmic}[1]
    \State {\bf Initialisation}:    
    \Statex $t\gets 0$; Initialise $P_t$ uniformly at random from $\mathcal{X}^\lambda$.
    \State {\bf Variation and Selection}:
    \For{$i=1 \ldots \lambda$}
    \Statex Sample $P_{t+1}(i) \sim D(P_t)$ 
    \EndFor
    \State $t\gets t+1$; Continue at \ref{variation}
  \end{algorithmic}
\end{algorithm}

In contrast to classical fitness-level theorems, the level-based
theorem (Theorem~\ref{thm:level-based}) only assumes a partition
$(A_1,\ldots,A_{m+1})$ of the search space $\mathcal{X}$, and not an
$f$-based partition (see Definition~\ref{def:f-based-partition}).
Each of the sets $A_j, j\in[m+1]$ is called a \emph{level}, and the
symbol $A_j^+:=\bigcup_{i=j+1}^{m+1} A_i$ denotes the set of search
points above level $A_j$. Given a constant $\gamma_0\in(0,1)$, a
population $P\in\mathcal{X}^\lambda$ is considered to be at level
$A_j$ with respect to $\gamma_0$ if $|P\cap A_{j-1}^+|\geq
\gamma_0\lambda$ and $|P\cap A_{j}^+|< \gamma_0\lambda$ meaning that
at least a $\gamma_0$ fraction of the population is in level $A_j$ or
higher.

\begin{theorem}[\cite{Lehre2014}]\label{thm:level-based}
  Given any partition of a finite set $\mathcal{X}$ into $m$
  non-overlapping subsets $(A_1,\dots,A_{m+1})$, define $T :=
  \min\{t\lambda \mid |P_t\cap A_{m+1}|>0\}$ to be the first point in
  time that elements of $A_{m+1}$ appear in $P_t$ of
  Algorithm~\ref{alg:pop}. If there exist parameters $z_1,\dots,z_m,
  z_*\in(0,1]$, $\delta>0$, and a constant $\gamma_0 \in (0,1)$ such
  that for all $j\in[m]$, $P\in \mathcal{X}^\lambda$, $y\sim D(P)$ and
  $\gamma \in (0,\gamma_0]$ it holds
  \begin{description}
  \item[(C1)] $\displaystyle
    \Pr\left( y\in A_j^+
      \mid |P\cap A_{j-1}^+|\geq\gamma_0\lambda   
    \right)\geq z_j\geq z_*$
  \item[(C2)] $\displaystyle
    \Pr\left( y\in A_j^+
      \mid |P\cap A_{j-1}^+|\geq\gamma_0\lambda\text{ and }
      |P\cap A_{j}^+|\geq \gamma \lambda   
    \right)\geq (1+\delta)\gamma$,\text{ and}
  \item[(C3)] $\displaystyle \lambda \geq
    \frac{2}{a}\ln\left(\frac{16m}{ac\varepsilon z_*}\right)$
    with $\displaystyle a = \frac{\delta^2 \gamma_0}{2(1+\delta)}$,
    $\varepsilon = \min\{\delta/2,1/2\}$
    and $c = \varepsilon^4/24$
  \end{description}
  then
  \begin{align*}
    \expect{T} \leq \frac{2}{c\varepsilon}\left(m\lambda(1+\ln(1+c\lambda)) +\sum_{j=1}^{m}\frac{1}{z_j}\right).
  \end{align*}
\end{theorem}

The theorem provides an upper bound on the expected optimisation time
of Algorithm \ref{alg:pop} if it is possible to find a partition
$(A_1,\ldots, A_{m+1})$ of the search space $\mathcal{X}$ and
accompanying parameters $\gamma_0, \delta, z_1,\ldots,z_m,z_*$ such
that conditions (C1), (C2), and (C3) are satisfied. Condition (C1)
requires a non-zero probability $z_j$ of creating an individual in
level $A_{j+1}$ or higher if there are already at least
$\gamma_0\lambda$ individuals in level $A_j$ or higher. In typical
applications, this imposes some conditions on the variation
operator. The condition is analogous to the probability $s_j$ in the
artificial fitness level technique. Condition (C2) requires that if in
addition there are $\gamma\lambda$ individuals at level $A_{j+1}$ or
better, then the probability of producing an individual in level
$A_{j+1}$ or better should be larger than $\gamma$ by a multiplicative
factor $1+\delta$. In typical applications, this imposes some
conditions on the strength of the selective pressure in the
algorithm. Finally, condition (C3) imposes minimal requirements on the
population size in terms of the parameters above.

As an example application of the level-based theorem, the
$(\mu,\lambda)$ EA is analysed, which is the non-elitist variant of
the $(\mu+\lambda)$ EA shown in Algorithm \ref{alg:mulambdaea}. The
two algorithms differ in the selection step (line \ref{lineSelection})
where the new population $P_{t+1}$ in $(\mu,\lambda)$ EA is chosen as
the best $\mu$ individuals out of $\{y_1,\ldots,y_\lambda\}$ and
breaking ties uniformly at random. While the $(\mu+\lambda)$ EA always
retains the best $\mu$ individuals in the population (hence the name
elitist), the $(\mu,\lambda)$ EA always discards the old individuals
$x^{(1)},\ldots,x^{(\mu)}$.

At first sight, it may appear as if the $(\mu,\lambda)$ EA cannot be
expressed in the form of Algorithm \ref{alg:pop}. The $\mu$
individuals $x^{(1)},\ldots,x^{(\mu)}$ that are kept in each generation are
not independent due to the inherent sorting of the offspring. However,
taking a different perspective, the population of the algorithm at
time $t$ could also be interpreted as the $\lambda$ offspring
$y^{(1)},\ldots,y^{(\lambda)}$. In this alternative interpretation, the new
population is now created by sampling uniformly at random among the
$\mu$ best individuals in the population, and applying the mutation
operator. The operator $D$ in Algorithm \ref{alg:pop} can now be
defined as in Algorithm~\ref{alg:d-mucommalambda}.

\begin{algorithm}
  \caption{\label{alg:d-mucommalambda}Operator $D$ corresponding to $(\mu,\lambda)$ EA}
  \begin{algorithmic}[1]
    \State {\bf Selection}:
    \Statex Sort the population $P_t=(y^{(1)},\ldots,y^{(\lambda)})$ such that
    $f(y^{(1)})\geq f(y^{(2)})\geq \ldots \geq f(y^{(\lambda)})$.
    \Statex Select $x$ uniformly at random among $\{y^{(1)},\ldots,y^{(\mu)}\}$.
    \State {\bf Variation (mutation)}:
    \Statex Create $x'$ by flipping each bit in $x$ with probability $\chi/n$.
    \State \Return $x'$
  \end{algorithmic}
\end{algorithm}

The following lemma will be useful when estimating the probability
that the mutation operator does not flip any bit positions.
\begin{lemma}\label{lemma:bound}
  For any $\delta\in(0,1)$ and $\chi>0$, if $n\geq
  (\chi+\delta)(\chi/\delta)$ then
  \begin{align*}
    \left(1-\frac{\chi}{n}\right)^n \geq (1-\delta)e^{-\chi}.
  \end{align*}
\end{lemma}
\begin{proof}
  Note first that $\ln(1-\delta)<-\delta$, hence
  \begin{align*}
    \left(\frac{n}{\chi}-1\right)(\chi-\ln(1-\delta))
    \geq n+\frac{n\delta}{\chi}-(\chi+\delta) \geq n.
  \end{align*}
  By making use of the fact that $(1-1/x)^{x-1}\geq 1/e$ and
  simplifying the exponent $n$ as above
  \begin{align*}
    \left(1-\frac{\chi}{n}\right)^n 
    \geq \left[\left(1-\frac{\chi}{n}\right)^{(n/\chi)-1}\right]^{\chi-\ln(1-\delta)}
    \geq (1-\delta)e^{-\chi}.
  \end{align*}
\end{proof}

The expected optimisation time of the $(\mu,\lambda)$~EA on \ONEMAX
can now be expressed in terms of the mutation rate $\chi/n$ and the
problem size $n$ assuming some constraints on the population sizes
$\mu$ and $\lambda$. The theorem is valid for a wide range of
mutation rates $\chi/n$. In the classical setting of $\chi=1$, the
expected optimisation time reduces to $O(n\lambda\ln\lambda)$.

\begin{theorem}
  The expected optimisation time
  of the $(\mu,\lambda)$ EA with bitwise mutation rate $\chi/n$ where
  $\chi\in(0,n/2)$, and population sizes $\mu$ and $\lambda$
  satisfying for any constant $\delta\in(0,1)$
  \begin{align*}
  \frac{\lambda}{\mu} 
    \geq \left(\frac{1+\delta}{1-\delta}\right) e^{\chi},
\quad\text{
                and }\quad
  \lambda  \geq \frac{4}{\delta^2e^\chi}
                \ln\left(\frac{24576n(n+1)}{\delta^7\chi}\right)
  \end{align*}
  on \ONEMAX is for any $n\geq (\chi+\delta)/(\chi/\delta)$ no more
  than
  \begin{align*}
    \frac{1536n}{\delta^5}\left(\lambda\ln(\lambda)+\frac{e^\chi\ln(n+2)}{\chi(1-\delta)}\right)+O(n\lambda).
 \end{align*}
\end{theorem}
\begin{proof}  
  Apply the level-based theorem with the same $m:=n+1$ partitions
  as in the proof of Theorem~\ref{thm:flevels-onemax}.  Since the
  parameter $\delta$ is assumed to be some constant $\delta\in(0,1)$,
  it also holds that the parameters $a,\varepsilon,$ and $c$ are
  positive constants. The parameters $\gamma_0,
  z_1,\ldots,z_m,$ and $z_*$ will be chosen later.

  To verify that conditions (C1) and (C2) hold for any $j\in[m]$, it
  is necessary to estimate the probability that operator $D$ produces
  a search point $x'$ with $j+1$ one-bits when applied to a population
  $P$ containing at least $\gamma_0\lambda$ individuals, each having
  at least $j$ one-bits (formally $|P\cap A_{j-1}^+|\geq
  \gamma_0\lambda$). Such an event is called a \emph{successful sample}.
  
  Condition (C1) asks for bounds $z_j$ for each $j\in[m]$ on the
  probability that the search point $x'$ returned by Algorithm
  \ref{alg:d-mucommalambda} contains $j+1$ one-bits.  First chose the
  parameter setting $\gamma_0:=\mu/\lambda$. This parameter setting is
  convenient, because the selection step in
  Algorithm~\ref{alg:d-mucommalambda} always picks an individual $x$
  among the best $\mu=\gamma_0\lambda$ individuals in the population.
  By the assumption that $|P\cap A_{j-1}^+|\geq \gamma_0\lambda$, the
  algorithm will always select an individual $x$ containing at least
  $j+i$ one-bits for some non-negative integer $i\geq 0$.

  Assume without loss of generality, that the first $j$ bit-positions
  in the selected individual $x$ are one-bits, and let $k, j<k\leq n$,
  be any of the other bit positions. If there is a zero-bit in
  position $k$ or if $i\geq 2$, then a successful sample occurs if the
  mutation operator flips only bit position $k$. If there is a one-bit
  in position $k$, and if $i=1$, then the step is still successful if
  the mutation operator flips none of the bit positions. Since the
  probability of not flipping a position is higher than the
  probability of flipping a position, i.e., $1-\chi/n\geq \chi/n$, the
  probability of a successful sample is therefore in both cases at least
  \begin{align}
    (n-j) (\chi/n)(1-\chi/n)^{n-1}.    
  \end{align}
  By Lemma~\ref{lemma:bound}, the probability above is at least
  $z_j:=(n-j)(\chi/n)e^{-\chi}(1-\delta).$ The parameter $z_*$ is
  chosen to be the minimal among these probabilities,
  i.e. $z_*:=(\chi/n)e^{-\chi}(1-\delta)$.

  Condition (C2) assumes in addition that $\gamma\lambda<\mu$
  individuals have fitness $j+1$ or higher. In this case, it suffices
  that the selection mechanism picks one of the best $\gamma\lambda$
  individuals among the $\mu$ individuals, and that none of the bits
  are mutated in the selected individual. The probability of this
  event is at least
  \begin{align*}
    \frac{\gamma\lambda}{\mu} (1-\chi/n)^{n}\geq
    \frac{\gamma\lambda}{\mu} e^{-\chi}(1-\delta)
  \end{align*}
  Hence, to satisfy condition (C2), it suffices to require that
  \begin{align*}
    \frac{\gamma\lambda}{\mu} \exp(-\chi)(1-\delta)
    \geq \gamma(1+\delta),
  \end{align*}  
  which is true whenever 
  \begin{align*}
    \frac{\lambda}{\mu} 
    \geq \left(\frac{1+\delta}{1-\delta}\right) e^{\chi}.
  \end{align*}

  To check condition (C3), notice that 
  $\varepsilon = \delta/2$, and
  $c = \delta^4/384$, hence
  \begin{align*}
    a & = \frac{\delta^2(\lambda/\mu)}{2(1+\delta)}\geq
    \frac{\delta^2e^{\chi}}{2(1-\delta)},\text{ and}\\
    ac\varepsilon z_* & \geq \frac{\delta^2\chi}{2n}c\varepsilon = \frac{\delta^7\chi}{1536n}
  \end{align*}
  Condition (C3) is now satisfied, because the population size
  $\lambda$ is required to fulfil
  \begin{align*}
    \frac{2}{a}\ln\left(\frac{16m}{ac\varepsilon z_*}\right)
    & \leq \frac{4(1-\delta)}{\delta^2e^\chi}
      \ln\left(\frac{24576mn}{\delta^7\chi}\right)
      \leq \lambda
  \end{align*}
  All conditions are satisfied, and the theorem follows.
\end{proof}

\subsection{Conclusions}

The artificial fitness levels method was first described by Wegener
\cite{Wegener2002}.  The original method was designed for the
achievement of upper bounds on the runtime of stochastic search
heuristics using only one individual such as the (1+1)~EA.  Since
then, several extensions of the method have been devised for the
analysis of more sophisticated algorithms.  Sudholt introduced the
method presented in Section~\ref{AFL-LowerBounds} for the obtainment
of lower bounds on the runtime \cite{Sudholt2010}. In an early study,
\cite{Witt2006} used a potential function that generalises the fitness
level argument of \cite{Wegener2002} to analyse the ($\mu$+1)~EA. His
analysis achieved tight upper bounds on the runtime of the
($\mu$+1)~EA on \LEADINGONES and \ONEMAX by waiting for a sufficient
amount of individuals of the population to \emph{take over} a given
fitness level $A_i$ before calculating the probability to reach a
fitness level of higher fitness. Chen et al. extended the analysis
to offspring populations by analysing the  ($N$+$N$)~EA, also taking 
into account the take over process \cite{Chen2009}.
Lehre introduced a general fitness-level method for
arbitrary population-based EAs with non-elitist selection mechanisms
and unary variation operators \cite{Lehre2011}. This technique was later generalised
further into the level-based method presented in Section
\ref{AFL-NonElitist} \cite{Lehre2014}. The method allows the analysis
of sophisticated non-elitist heuristics such as genetic algorithms
equipped with mutation, crossover and stochastic selection mechanisms,
both for classical as well as noisy and uncertain optimisation
\cite{DangLehre2015Noise}.

\section{Drift Analysis}

Drift analysis is a very flexible and powerful tool that is widely
used in the analysis of stochastic search algorithms.  The high level
idea is to predict the long term behaviour of a stochastic process by
measuring the expected progress towards a target in a single step.
Naturally, a measure of progress needs to be introduced, which is
generally called a \emph{distance function}.  Given a random variable
$X_k$ representing the current state of the process at step $k$, over
a finite set of states $S$, a distance function $d:S \rightarrow
\mathbb{R}_0^+$ is defined such that $d(X_k) = 0$ if and only if $X_k$
is a target point (e.g., the global optimum).  Drift analysis aims at
deriving the expected time to reach the target by analysing the
decrease in distance in each step, i.e., $d(X_{k+1}) - d(X_k)$. The
expected value of this decrease in distance, $\Delta_k =
\expect{d(X_{k+1}) - d(X_k)\mid X_k}$ is called the \emph{drift}. See
Figure~\ref{fig:drift} for an illustration.  If the initial distance
from the target is $d(X_0)$ and a bound on the drift $\Delta$ (i.e.,
the expected improvement in each step) is known, then bounds on the
expected runtime to reach the target may be derived.

\begin{figure}[t]
\begin{center}
        \includegraphics[width=7cm]{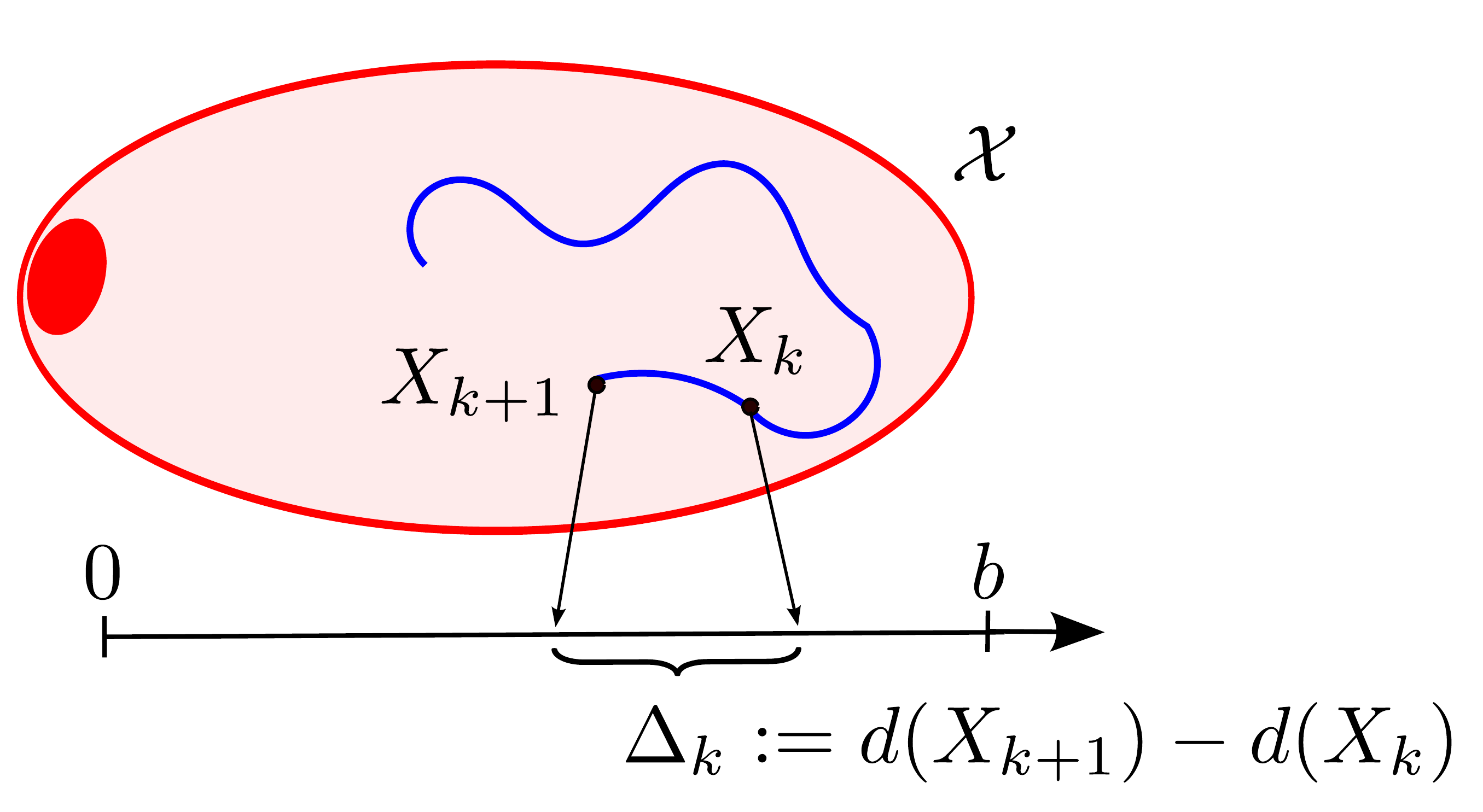}
      \end{center}
      \caption{An illustration of the drift $\Delta$ at time step $k$
        of a process represented by the random variable $X$ and a
        distance function $d$.}
\label{fig:drift}       
\end{figure}

\subsection{Additive Drift Theorem}

The additive drift theorem was introduced to the field of evolutionary
computation by He and Yao \cite{HeYao2001}.  The theorem allows to
derive both upper and lower bounds on the runtime of stochastic search
algorithms.  Consider a distance function $Y_k = d(X_k)$ indicating
the current distance, at time $k$, of the stochastic process from the
optimum.  The theorem simply states that if at each time step $k$, the
drift is \emph{at least} some value $-\varepsilon$ (i.e., the process
has moved closer to the target) then the expected number of steps to
reach the target is \emph{at most} $Y_0/\varepsilon$. Conversely if
the drift in each step is \emph{at most} some value $-\varepsilon$,
then the expected number of steps to reach the target is \emph{at
  least} $Y_0/\varepsilon$.

\begin{theorem}[Additive Drift Theorem] \label{additivedrift} Given a
  stochastic process $X_1, X_2,\ldots $ over an interval
  $[0,b]\subset\mathbb{R}$ and a distance function $d:S \rightarrow
  \mathbb{R}_0^+$ such that $d(X_k)=0$
  if and only if $X$ contains the target. Let $Y_k = d(X_k)$ for all $k$, define $T:=\min\{k\geq 0\mid Y_k= 0\}$, and assume $\expect{T}<\infty$.\\
  Verify the following conditions:
    \begin{itemize}
      \setlength{\itemindent}{2em}
    \item[(C1+)] $\forall k\quad \expect{Y_{k+1}-Y_k\mid Y_k>0}\leq-\varepsilon$
    \item[(C1$-$)] $\forall k\quad  \expect{Y_{k+1}-Y_k\mid Y_k>0}\geq-\varepsilon$
    \end{itemize}    
  Then,  
    \begin{enumerate}
    \item If (C1+) holds for an $\varepsilon>0$, then
      $\expect{T\mid Y_0}\leq b/\varepsilon$.
    \item If (C1$-$) holds for an $\varepsilon>0$, then
      $\expect{T\mid Y_0}\geq Y_0/\varepsilon$.
    \end{enumerate}
\end{theorem}

%
%
%
%
%

\begin{figure}[t]
\begin{center}
        \includegraphics[width=11cm]{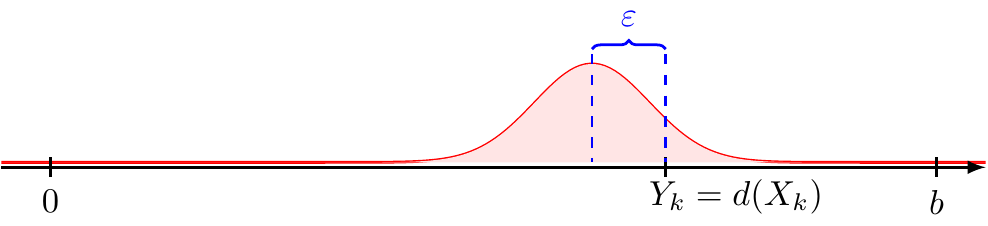}
      \end{center}
      \caption{An illustration of the condition of the Additive Drift
  Theorem. If the expected distance to the optimum decreases of
  \emph{at least} $\varepsilon$ at each step (i.e., condition C1+),
  then an upper bound on the runtime is achieved. If the distance
  decreases of \emph{at most} $\varepsilon$ at each step (i.e.,
  condition C1-), then a lower bound on the runtime is obtained.}
\label{fig:distribution}       
\end{figure}

An Example application of the additive drift theorem follows
concerning the (1+1)~EA for plateau blocks of functions of unitation
of length $m$ positioned such that $k> n/2 + \varepsilon n$.

\begin{theorem}
  The expected runtime of the (1+1)-EA for a \emph{plateau block} of
  length $m$ ending at position $k> n/2 + \varepsilon n$ is
  $\Theta(m)$.
\end{theorem}

\begin{proof}
  The additive drift theorem will be applied to derive both upper and
  lower bounds on the expected runtime.  The starting point is a
  bitstring $X_0$ with $m+k$ zeroes and the target point is a
  bitstring $X_t$ with $k$ zeroes.  Choose to use the natural distance
  function $Y_t = d(X_t) := n -|X_t|$ that counts the number of zeroes
  in the bitstring.  Subtract $k$ from the distance such that target
  points with $k$ zeroes have distance 0 and the initial point has
  distance $m$.  As long as points on the plateau are generated, they
  will be accepted because all plateau points have equal fitness.
  Given that each bit flips with probability $1/n$, and at each step
  the current search point has $Y_t$ zeroes and $n-Y_t$ ones, the
  drift is
\[
 \Delta_t := \expect{Y_t - Y_{t+1}\mid Y_t>0} = \frac{Y_t}{n}  - \frac{n-Y_t}{n} = \frac{2 \cdot Y_t}{n} -1
\]
A lower bound on the drift is obtained by considering that as long as
the end of the plateau has not been reached there are always at least
$k$ zeroes that may be flipped (i.e., $Y_t \geq k$). Accordingly for
an upper bound, at most $m+k$ zeroes may be available to be flipped
(i.e., $Y_t \leq m+k$). Hence,
\[
 \frac{2k}{n} -1 \leq \Delta_t \leq \frac{2(m+k)}{n} -1
\]
Then by additive drift analysis (Theorem~\ref{additivedrift}),
\[
 \expect{T \mid Y_0} \leq \frac{m}{(2k)/n -1} = \frac{mn}{2k-n} = O(m)
\]
and 
\[
  \expect{T \mid Y_0} \geq \frac{m}{2(m+k)/n -1} = \frac{mn}{2(m+k)-n} = \Omega(m)
\]
where the last equalities hold as long as $k > n/2 + \varepsilon n$.
\end{proof}
Note again that if the plateau block is followed by a gap block, then
an upper bound on the expected time to optimise both blocks is
achieved by multiplying the upper bounds obtained for each block. This
is necessary because points in the gap will not be accepted by the
(1+1)~EA.

\subsection{Multiplicative Drift Theorem} \label{sec:multiplicative}
In the additive drift theorem the worst case decrease in distance is
considered.  If the expected decrease in distance changes considerably
in different areas of the search space, then the estimate on the drift
may be too pessimistic for the obtainment of tight bounds on the
expected runtime.

Drift analysis of the (1+1)~EA for the classical \ONEMAX function will
serve as an example of this problem. Since the global optimum is the
all-ones bitstring and the fitness increases with the number of ones a
natural distance function is $Y_t = d(X_t) = n - \ONEMAX(X_t)$ which
simply counts the number of zeroes in the current search point.  Then
the distance will be zero once the optimum is found.  Points with less
one-bits than the current search point will not be accepted by the
algorithm because of their lower fitness.  So the drift is always
positive, i.e., $\Delta_t \geq 0$ and the amount of progress is the
expected number of ones gained in each step.  In order to find an
upper bound on the runtime, a lower bound on the drift is needed
(i.e., the worst case improvement).  Such worst case occurs when the
current search point is optimal except for one 0-bit.  In this case
the maximum decrease in distance that may be achieved in a step is
$Y_t - Y_{t+1}=1$ and to achieve such progress it is necessary that
the algorithm flips the zero into a one and leaves the other bits
unchanged. Hence, the drift is
\[
 \Delta_t \geq 1 \cdot \frac{1}{n} \left(1 - \frac{1}{n}\right)^{n-1} \geq \frac{1}{en} := \varepsilon
\]
Since the expected initial distance is $\expect{Y_0} = n/2$ due to
random initialisation, the drift theorem yields
\[
 \expect{T \mid Y_0} \leq \frac{\expect{Y_0}}{\varepsilon} = \frac{n/2}{1/(en)} = e/2 \cdot n^2 = O(n^2)
\]

In Section~\ref{sec:fitnesslevels} it was proven that the runtime of
the (1+1)~EA for \ONEMAX is $\Theta(n\ln n)$, hence a bound of
$O(n^2)$ is not tight.  The reason is that on functions such as
\ONEMAX the amount of progress made by the algorithm depends crucially
on the distance from the optimum.  For \ONEMAX in particular, larger
progress per step is achieved when the current search point has many
zeroes that may be flipped.  As the algorithm approaches the optimal
solution the amount of expected progress in each step becomes smaller
because search points have increasingly more one-bits than zero-bits
in the bitstring. In such cases a distance function that takes into
account these properties of the objective function needs to be used.
For \ONEMAX a correct bound is achieved by using a distance function
that is logarithmic in the number of zeroes $i$, i.e., $Y_t = d(X_t)
:= \ln (i+1)$ where a $1$ is added to $i$ in the argument of the
logarithm such that the global optimum has distance zero (i.e.,
$\ln(1)=0$).  With such distance measure, the decrease in distance
when flipping a zero and leaving the rest of the bitstring unchanged
is
\[ 
\ln(i+1) - \ln(i) = \ln \left(1 + \frac{1}{i}\right) \geq \frac{1}{2i} 
\]
where the last inequality holds for all $i \geq 1$.  Since it is
sufficient to flip a zero and leave everything else unchanged to
obtain an improvement, the drift is
\[
 \Delta_t \geq \frac{i}{en} \cdot \frac{1}{2i} = \frac{1}{2en} := \varepsilon
\]

Given that the maximum possible distance is $Y_0 \geq \ln(n+1)$, the
drift theorem yields
\[
 \expect{T} \leq \frac{Y_0}{1/(2en)} = 2en \cdot \ln(n+1) = O(n \ln n).
\]

The multiplicative drift theorem was introduced as a handy tool to deal with situations as the one described above where the amount of progress depends
on the distance from the target.

\begin{theorem}[Multiplicative Drift Theorem~\cite{Doerr2010}]
  Let $\{X_{t}\}_{t\in\mathbb{N}_0}$ be random variables describing a
  Markov process over a finite state space $S \subseteq \mathbb{R}$.
  Let $T$ be the random variable that denotes the earliest point in
  time $t \in \mathbb{N}_0$ such that $X_{t} = 0$.  If there exist
  $\delta, c_{\min}, c_{\max} > 0$ such that for all $t < T$,
	\begin{enumerate}
		\item $\expect{X_{t}-X_{t+1} \mid X_{t}} \geq \delta X_{t}$ and
		\item $c_{\min} \leq X_{t} \leq c_{\max}$,
	\end{enumerate}
 then
		\[\expect{T} \leq \frac{2}{\delta} \cdot \ln\left(1+\frac{c_{\max}}{c_{\min}}\right)\]
\end{theorem}

The following derivation of an upper bound on the runtime of the
(1+1)~EA for linear blocks illustrates the multiplicative drift
theorem.

\begin{theorem}
  The expected time for the (1+1)-EA to optimise a linear unitation
  block of length $m$ ending at position $k$ is $O(n \ln ((m+k)/k))$
\end{theorem}

\begin{proof}
  Let $X_t$ be the number of zero-bits in the bitstring at time step
  $t$, representing the distance from the end of the linear block.  By
  remembering that increases in distance are not accepted due to
  elitism, the expected decrease in distance at time step can be
  bounded by
\[
 \expect{X_{t+1}|X_t} \leq X_t - 1 \cdot \frac{X_t}{en} = X_t\left(1 - \frac{1}{en}\right)
\]
simply by considering that if a zero-bit is flipped and nothing else
then the distance decreases by 1. Then the drift is:
\[
\expect{X_{t}- X_{t+1}|X_t} \geq X_t - X_t\left(1 - \frac{1}{en}\right) = \frac{1}{en} X_t := \delta X_t 
\]
By fixing $k=c_{\min} \leq X_t \leq c_{\max} = m+k$ the multiplicative
drift theorem yields
\[
  \expect{T} \leq \frac{2}{\delta} \cdot \ln\left(1+\frac{c_{\max}}{c_{\min}}\right) = {2en} \ln(1+(m+k)/k) 
= O(n\ln ((m+k)/k))
\]
\end{proof}
By fixing $1=c_{\min} \leq X_t \leq c_{\max} = n$ an $O(n \ln n)$
bound on the expected runtime of the (1+1)~EA for \ONEMAX is achieved.

\subsection{Variable Drift Theorem}\label{sec:variable}

The multiplicative drift theorem is applicable when the drift of a
stochastic process is linear with respect to the current
position. However, in some stochastic processes, the drift is
non-linear in the current position, i.e.,
\begin{align}
  \expect{X_t-X_{t+1} \mid  X_t\ge \xmin} \ge h(X_t)
\end{align}
for some function $h$. The following variable drift theorem provides
bounds on the expectation and the tails of the hitting time
distribution of such processes, given some assumptions about the
function $h$.

\begin{theorem}[Corollary 1 in \cite{LehreWitt2013TailDrift}]\label{thm:variable-drift}
  Let $(X_t)_{t\ge 0}$, be a stochastic process over some state space
  $S\subseteq \{0\}\cup [\xmin,\xmax]$, where $\xmin\ge 0$. Let
  $h\colon [\xmin,\xmax]\to\mathbb{R}^+$ be a differentiable function.
  Then the following statements hold for the first hitting time
  $T:=\min\{t\mid X_t=0\}$.
  \begin{enumerate}
  \item[(i)] If $\expect{X_t-X_{t+1} \mid  X_t\ge \xmin} \ge h(X_t)$ and
    $h'(x)\geq 0$, then 
    \begin{align*}
      \expect{T\mid X_0} \le \frac{\xmin}{h(\xmin)} + \int_{\xmin}^{X_0} \frac{1}{h(y)} \,\mathrm{d}y.
    \end{align*}

  \item[(ii)] If $\expect{X_t-X_{t+1} \mid  X_t\ge \xmin} \le h(X_t)$ and
    $h'(x)\leq 0$, then 
    \begin{align*}
      \expect{T\mid X_0} \ge \frac{\xmin}{h(\xmin)} + \int_{\xmin}^{X_0} \frac{1}{h(y)} \,\mathrm{d}y.
    \end{align*}
		
  \item[(iii)] If $\expect{X_t-X_{t+1} \mid  X_t\ge \xmin} \ge h(X_t)$ and 
    $h'(x)\geq \lambda$ for some $\lambda>0$, then
    \begin{align*}
      \Prob{T\ge t \mid X_0} < \exp\left(-\lambda \left(t-\frac{\xmin}{h(\xmin)}-\int_{\xmin}^{X_0} \frac{1}{h(y)}\,\mathrm{d}y\right)\right).
    \end{align*}
  \item[(iv)]
    If $\expect{X_t-X_{t+1} \mid  X_t\ge \xmin} \le h(X_t)$
    and $h'(x)\leq -\lambda$ for some $\lambda>0$, then
    \begin{align*}
      \Prob{T < t \mid X_0>0} < \frac{e^{\lambda t}-e^{\lambda}}{e^\lambda-1} \exp\left(-\frac{\lambda\xmin}{h(\xmin)}-\int_{\xmin}^{X_0} \frac{\lambda}{h(y)}\,\mathrm{d}y\right).
    \end{align*}
  \end{enumerate}
\end{theorem}

To illustrate the variable drift theorem, an upper bound on the
optimisation time of the (1+1) EA on the class of linear functions
with bounded coefficients will be derived. More formally, this class
of functions  contains any function of the form
\begin{align*}
  f(x) := \sum_{i=1}^{n} w_ix_i,
\end{align*}
with bounded, positive coefficients $w_1,\ldots, w_n\in(\wmin,\wmax)$
where $0<\wmin<\wmax.$

The drift function $h$ in this example turns out to be linear, hence
the multiplicative drift theorem could have been applied instead.

\begin{theorem}
  The expected optimisation time of the (1+1) EA on linear functions is
  less than
  $t(n):=en(\ln(n) + \ln(\wmax/\wmin)+1)$, and the probability that
  the optimisation time 
  exceeds $t(n)+ren$ for any $r\geq 0$ is no more than $e^{-r}$.
\end{theorem}
\begin{proof}
  Define the distance $X_t$ at time
  $k$ to be the function value that ``remains'' at time $k$, i.e.,
    \begin{align*}
      X_t := \left(\sum_{i=1}^n w_i\right) - \left(\sum_{i=1}^n w_i x^{(t)}_i\right) = \sum_{i=1}^nw_i\left(1-x^{(t)}_i\right),
    \end{align*}
    where $x_i^{(t)}$ is the $i$-th bit in the current search point at
    time $t$. For any $i\in[n]$, assume that the mutation operator
    flipped only bit position $i$, and no other bit positions, an
    event denoted by the symbol $\mathcal{E}_i$.  If
    $x_i^{(t)}=0$, then bit position $i$ flipped from 0 to 1 and the
    distance reduced by $w_i$. Otherwise, if $x_i^{(t)}=1$, then bit
    position $i$ flipped from 1 to 0, the new search point was not
    accepted, and the distance reduced by 0. Hence, the distance
    always reduces by $w_i(1-x_{i}^{(t)})$ when event $\mathcal{E}_i$
    occurs.  Using the law of total probability, and noting that the
    distance can never increase, one obtains
    \begin{align*}
      \expect{X_t-X_{t+1}\mid X_t=r} 
      &\geq \sum_{i=1}^n
      \Prob{\mathcal{E}_i\mid X_t=r}
      \expect{X_t-X_{t+1} \mid \mathcal{E}_i\wedge X_t=r}\\
      &\geq 
      \left(\frac{1}{n}\right)
      \left(1-\frac{1}{n}\right)^{n-1}\sum_{i=1}^n w_i\left(1-x_i^{(t)}\right)
      \geq \frac{r}{en}.
    \end{align*}
    Therefore 
    $\expect{X_t-X_{t+1}\mid X_t}\geq h(X_t) $ for the function
    $h(x)=x/en$ which has derivative $h'(x)=1/en$.  By
    Theorem~\ref{thm:variable-drift} part (i), it follows that
    \begin{align*}
      \expect{T\mid X_0} 
      & \leq \frac{\xmin}{h(\xmin)}+\int_{\xmin}^{X_0}\frac{1}{h(y)} \,\mathrm{d}y\\
      & =    \frac{\wmin }{h(\wmin)}+\int_{\wmin}^{n\wmax}\frac{en}{y} \,\mathrm{d}y\\
      & =     en + en(\ln(n\wmax)-\ln(\wmin))\\
      & =     en(\ln(n) + \ln(\wmax/\wmin)+1) =: t(n).
    \end{align*}
    Furthermore, Theorem~\ref{thm:variable-drift} part (iii) with
    $\lambda:=1/en$ gives for any $r\geq 0$
    \begin{align*}
      \Prob{T\geq t(n)+enr}
      \leq e^{-r}.
    \end{align*}
    
\end{proof}

\subsection{Negative-Drift Theorem} \label{sec:negativedrift} The
drift theorems presented in previous subsections are designed to prove
polynomial bounds on the runtime of stochastic search algorithms.  For
this a positive drift towards the optimum is required.  On the other
hand, a negative drift indicates that the stochastic process moves
away from the optimum in expectation at each step.  In such a case it
is unlikely that the algorithm is efficient for the function it is
attempting to optimise.  Rather, an exponential lower bound on the
runtime could probably be proved. The negative-drift theorem is the
standard technique used for the purpose.

Apart from a negative drift, the theorem also requires a second
condition showing that large jumps are unlikely.  The intuitive reason
for this second condition is that if large jumps were possible, then
the algorithm could maybe be able to jump to the optimum even if in
expectation it is drifting away. For technical reasons also large
jumps heading away from the optimum need to be excluded (see
\cite{OlivetoWitt2012}).

\begin{theorem}[Negative-Drift Theorem] \label{th:negativedrift} Let
  $X_t$, $t\geq 0$ be the random variables describing a Markov process
  over the state space $S$ and denote the increase in distance in one
  step $D_t(i) := (X_{t+1} - X_t | X_t=i)$ for $i \in S$ and $t \geq
  0$.  Suppose there exists an interval $[a,b]$ of the state space and
  three constants $\delta, \varepsilon, r > 0$ such that for all $t
  \geq 0$ the following conditions hold:
    \begin{enumerate}
    \item $\Delta_t(i) = \expect{D_t(i)} \geq \varepsilon$ for $a < i < b$
    \item $\Prob{|D_t(i)|) = j} \leq \frac{1}{(1+\delta)^{j-r}}$ for $i>a$ and $j\geq 1$ 
    \end{enumerate}    
    Then there exists a constant $c^*$ such that for $T:=\min\{t\geq 0
    : X_t \leq a | X_0 \geq b \}$ it holds $Pr(T \leq 2^{c^*(b-a)}) =
    2^{-\Omega(b-a)}$.
  \end{theorem}

  \begin{figure}[t]
\centering

\begin{tikzpicture}[scale=5]
\draw[thick, line width=1.5] (0,0) node[left=3] {target} -- 
	node[solid,pos=0.225] {\tikz \filldraw (0,0) circle (2pt);}
	node[solid,pos=0.225, below=2] {$a$}
        node[solid,pos=0.6] {\tikz \filldraw(0,0) circle(2pt); }
	node[solid,pos=1] {\tikz \filldraw (0,0) circle (2pt);}
	node[solid,pos=1, below=2] {$b$}
		(1,0) node[right=10] {};
\draw[color=red, line width = 2pt, ->] (0.45,0.1) -- (0.55,0.1) 
node[above=3, color=black] {drift away from target} -- (0.7,0.1);
\draw[->,color=red, dotted] (0.6,0)  .. controls (0.5,-0.15) .. (0.4,0);
\draw[->,color=red, dotted] (0.6,0)  .. controls (0.55,-0.2) .. (0.5,0);
\draw[->,color=red, thick] (0.6,0)  .. controls (0.575,-0.3) .. (0.55,0);

\draw[->,color=red, dotted] (0.6,0)  .. controls (0.7,-0.15) .. (0.8,0);
\draw[->,color=red, dotted] (0.6,0)  .. controls (0.65,-0.2) .. (0.7,0);
\draw[->,color=red, thick] (0.6,0)  .. controls (0.625,-0.3) .. (0.65,0);
\path (0.75,-0.15) node[right=0.5]{no large jumps};
\path (0.75,-0.21) node[right=0.5]{towards target};
\draw[color=blue, line width = 1.5pt] (1.005,0) -- (1.4,0) node[right=3] {start};
\end{tikzpicture}

\caption{An illustration of the two conditions of the Negative-Drift Theorem.}
\label{fig:negativedrift}
\end{figure}
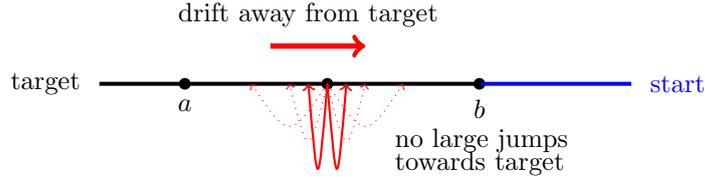

The following theorem is an example application showing exponential
runtime of the (1+1)~EA for the \NEEDLE function.

 \begin{theorem} \label{th:needle} 
 Let $\eta> 0$ be constant.  Then
  there is a constant $c>0$ such that with probability $1-
  2^{-\Omega(n)}$ the (1+1)~EA on $\text{{\sc Needle}}$ creates
  only search points with at most $n/2+\eta n$ ones in $2^{cn}$ steps.
\end{theorem}

\begin{proof}
  Apply the negative-drift theorem and set the interval $[a,b]$ as $a
  := n/2 - 2 \gamma n$ zeroes and $b:= n/2 - \gamma n$ zeroes, with
  $\gamma$ a positive constant.  By Chernoff bounds the probability
  that the initial random search point has less than $n/2 - \gamma n$
  zeroes is $e^{-\Omega(n)}$, implying that the algorithm starts
  outside the interval $[a,b]$ as desired. The remainder of the proof
  shows that the two conditions of the drift theorem hold.

  Since the interval is a plateau, all created points have the same
  fitness and are accepted.  Given that the current search point has
  $i$ zeroes and that each bit is flipped with probability $1/n$, the
  drift is
\[
 \Delta_t(i) = \frac{n-i}{n} - \frac{i}{n} = \frac{n-2i}{n} \geq 2 \gamma := \varepsilon
\]
where the last inequality is achieved because in the interval there
are always at least $b$ zeroes, i.e., $i \geq n/2 - \gamma n$.
Concerning the second condition for standard bit mutation, the
probability of flipping $j$ bits decreases exponentially with the
number of bits $j$
\[
 \Prob{|\Delta_t(i)| \geq j} \leq {n \choose j} \left(\frac{1}{n}\right)^j \leq \left(\frac{n^j}{j!}\right)^j \leq \frac{1}{j!} \leq \left(\frac{1}{2}\right)^{j-1}
\]
and the condition holds for $\delta=r=1$.  By the negative drift
theorem the optimum is found within $2^{c*(b-a)}=2^{cn}$ steps with
probability at most $2^{-\Omega(b-a)}=2^{-\Omega(n)}$.  
\end{proof}

The proof can be generalised to plateau blocks of functions of
unitation positioned such that $k+m < (1/2 - \varepsilon) n$.

\begin{theorem}
  The time for the (1+1)~EA to optimise a \emph{plateau block} of
  length $m$ at position such that $k+m < (1/2 - \varepsilon) n$ is at
  least $2^{\Omega(m)}$ with probability at least $1- 2^{-\Omega(m)}$.
\end{theorem}

\begin{proof}
  The proof follows the same arguments as the proof of
  Theorem~\ref{th:needle} by setting $b:=m+k$ zeroes and $a:=k$
  zeroes.  
\end{proof}

\subsection{Conclusions}
Drift analysis dates back to 1892 when it was applied for the analysis
of stability equilibria in ordinary differential equations
\cite{Lyapunov1892}.  The first use of drift techniques for the
runtime analysis of stochastic search heuristics was performed by
Sasaki and Hajek to analyse simulated annealing on instances of the
maximum matching problem \cite{SasakiHajek1988}.  Drift analysis was
applied in a considerable number of applications in evolutionary
computation after He and Yao introduced the additive drift technique
to the field \cite{HeYao2001}.  Several extensions for the analysis of
more sophisticated algorithms and processes have been since devised.
The multiplicative drift method introduced in
Section~\ref{sec:multiplicative} is due to Doerr et
al. \cite{Doerr2010}. An improved version introduced by Doerr and
Goldberg allows to derive bounds also on the probability to deviate
from the expected runtime. Witt recently introduced a multiplicative
drift theorem to achieve lower bounds on the runtime
\cite{Witt2012}. While the multiplicative drift theorem may only be
used when the drift is linear, a variable drift theorem was introduced
to deal with cases where the expected progress is non-linear
\cite{Johannsen2012} with respect to the current position. A general
variable drift theorem with tail bounds was introduced in
\cite{LehreWitt2013TailDrift}. This theorem subsumes most existing
drift theorems, and a special case of this theorem was given in
Section~\ref{sec:variable} (see Theorem~\ref{thm:variable-drift}).

Oliveto and Witt proposed the negative-drift theorem presented in
Section~\ref{sec:negativedrift} to derive exponential lower bounds on
the runtime of randomised search heuristics \cite{OlivetoWitt2011}.
This theorem was the main tool used in the first analysis of the
standard \emph{simple genetic algorithm (SGA)}
\cite{OlivetoWitt2014,OlivetoWitt2015}.  Finally, Lehre extended the
negative drift theorem to allow a systematic analysis of
population-based heuristics using non-elitist selection mechanisms
\cite{Lehre2010}.

\section{Final Overview}
This chapter has presented the most commonly used techniques for the
time complexity analysis of stochastic search heuristics.  Example
applications have been shown concerning simple evolutionary algorithms
on classical test problems.  The same techniques have allowed the
obtainment of time complexity results on more sophisticated function
classes, such as standard combinatorial optimisation problems.  The
reader is referred to \cite{NeumannWitt2010} for an overview of such
advanced results concerning evolutionary algorithms and ant colony
optimisation.  Several other techniques have been used to analyse
stochastic search heuristics, such as typical runs, family trees,
martingales, probability generating functions and branching processes.
For an introduction to these tools for the analysis of evolutionary
algorithms, the reader is referred to a recent extensive monograph
\cite{Jansen2013} which covers a more extensive set of techniques than
is possible in this book chapter.  Apart from EAs and ACO for discrete
single-objective optimisation, several other aspects of stochastic
search optimisation have been investigated theoretically, such as
simulated annealing, evolution strategies for continuous optimisation,
particle swarm optimisation, memetic algorithms, and multi-objective optimisation techniques.  These topics are overviewed in
\cite{AugerDoerr2010}. Significant results have been achieved recently
that have not yet been covered in monographs and edited book
collections.  Amongst these, certainly worth mentioning are the recent
considerable advances in black box optimisation
\cite{LehreWitt2012,DoerrWinzenTCS2014}, which may be regarded as a
complementary line of research to runtime analysis.  Rather than
determining the time required by a given heuristic for a problem
class, black box optimisation focuses on determining the best possible
performance achievable by any search heuristic for that class of
problems. A recent interesting variation to classical runtime analysis
is to determine the expected solution quality achieved by a stochastic
search heuristic if it is only allowed a predefined budget
\cite{JansenZarges2014a}. Fixed budget computation analyses are driven
by the consideration that in practical applications only a predefined
amount of resources are available and the hope is to use such
resources at best to achieve solutions of the highest quality. Recent
years have witnessed considerable advances in the theory of artificial
immune systems (AISs). Results concerning standard AISs, such as the
B-Cell algorithm, for classical combinatorial optimisation problems
have appeared \cite{JansenOlivetoZarges2011, JansenZarges2012}
together with analyses of sophisticated AIS operators such as
stochastic ageing mechanisms \cite{OlivetoSudholt2014}. Complexity
analyses of parallel evolutionary algorithms \cite{Sudholt2015} and
genetic programming \cite{NeumannOReillyWagner2011} have also recently
appeared. Finally, systematic work has been carried out in unifying
theories of evolutionary algorithms and population genetics
\cite{Paixao2015}.

\section{Acknowledgement}
The research leading to these results has received funding from the
European Union Seventh Framework Programme (FP7/2007-2013) under grant
agreement no 618091 (SAGE) and by the EPSRC under grant agreement no EP/M004252/1 (RIGOROUS).

\bibliographystyle{abbrv}
\bibliography{references}

\end{document}